\documentclass[sigconf]{acmart}
\AtBeginDocument{%
  }

\copyrightyear{2026}
\acmYear{2026}
\setcopyright{cc}
\setcctype{by}
\acmConference[WWW '26] {Proceedings of the ACM Web Conference 2026}{April 13--17, 2026}{Dubai, United Arab Emirates.}
\acmBooktitle{Proceedings of the ACM Web Conference 2026 (WWW '26), April 13--17, 2026, Dubai, United Arab Emirates}
\acmISBN{979-8-4007-2307-0/2026/04}
\acmDOI{10.1145/3774904.3792090}
\usepackage{algorithm}
\usepackage{algorithmic}

\usepackage{amssymb}
\usepackage{amsmath} 
\usepackage{booktabs}
\usepackage{multirow} 
\usepackage{siunitx}
\usepackage{array} 
\usepackage{amsthm}
\usepackage{longtable}

\newtheorem{theorem}{Theorem}
\newtheorem*{theorem*}{Theorem}
\newtheorem{lemma}{Lemma}[section]

\begin{document}

\title{Riemannian Liquid Spatio-Temporal Graph Network}

\author{Liangsi Lu}
\authornote{Both authors contributed equally to this research.}
\email{lu.liangsi.cn@gmail.com}
\orcid{0009-0006-2839-3901}
\affiliation{%
  % \department{School of Mathematics and Statistics}
  \institution{Guangdong University of Technology}
  \city{Guangzhou}
  % \state{Guangdong}
  \country{China}}

\author{Jingchao Wang}
\authornotemark[1]
\orcid{0000-0002-0099-539X}
\email{ethanwangjc@163.com}
\affiliation{%
  % \department{School of Computer Science}
  \institution{Peking University}
  \city{Beijing}
  \country{China}}

\author{Zhaorong Dai}
\orcid{0009-0008-5479-2580}
\email{13725167598@163.com}
\affiliation{%
  % \department{College of Economics and Management}
  \institution{South China Agricultural University}
  \city{Guangzhou}
  % \state{Guangdong}
  \country{China}}

\author{Hanqian Liu}
\orcid{0009-0000-7666-192X}
\email{15899789962@163.com}
\affiliation{%
  % \department{School of Mathematics (Zhuhai)}
  \institution{Sun Yat-Sen University}
  \city{Zhuhai}
  % \state{Guangdong}
  \country{China}}

\author{Yang Shi}
\authornote{Yang Shi is the corresponding author.}
\orcid{0009-0009-3928-7495}
\email{sudo.shiyang@gmail.com}
\affiliation{%
  % \department{School of Computer Science and Technology}
  \institution{Guangdong University of Technology}
  \city{Guangzhou}
  % \state{Guangdong}
  \country{China}}

\renewcommand{\shortauthors}{Lu et al.}

\begin{abstract}
Liquid Time-Constant networks (LTCs), a type of continuous-time graph neural network, excel at modeling irregularly-sampled dynamics but are fundamentally confined to Euclidean space. This limitation introduces significant geometric distortion when representing real-world graphs with inherent non-Euclidean structures (e.g., hierarchies and cycles), degrading representation quality. To overcome this limitation, we introduce the \textbf{R}iemannian \textbf{L}iquid \textbf{S}patio-\textbf{T}emporal \textbf{G}raph Network (RLSTG), a framework that unifies continuous-time liquid dynamics with the geometric inductive biases of Riemannian manifolds. RLSTG models graph evolution through an Ordinary Differential Equation (ODE) formulated directly on a curved manifold, enabling it to faithfully capture the intrinsic geometry of both structurally static and dynamic spatio-temporal graphs. Moreover, we provide rigorous theoretical guarantees for RLSTG, extending stability theorems of LTCs to the Riemannian domain and quantifying its expressive power via state trajectory analysis. Extensive experiments on real-world benchmarks demonstrate that, by combining advanced temporal dynamics with a Riemannian spatial representation, RLSTG achieves superior performance on graphs with complex structures. Project Page: \href{https://rlstg.github.io}{https://rlstg.github.io}
\end{abstract}

\begin{CCSXML}
<ccs2012>
   <concept>
       <concept_id>10010147.10010178</concept_id>
       <concept_desc>Computing methodologies~Artificial intelligence</concept_desc>
       <concept_significance>500</concept_significance>
       </concept>
   <concept>
       <concept_id>10002951.10003227</concept_id>
       <concept_desc>Information systems~Information systems applications</concept_desc>
       <concept_significance>500</concept_significance>
       </concept>
 </ccs2012>
\end{CCSXML}

\ccsdesc[500]{Computing methodologies~Artificial intelligence}
\ccsdesc[500]{Information systems~Information systems applications}

\keywords{Riemannian Manifolds,
Neural ODEs,
Spatio-Temporal Graphs
}


\maketitle

\begin{figure}[!t]
    \centering
    \includegraphics[width=\linewidth]{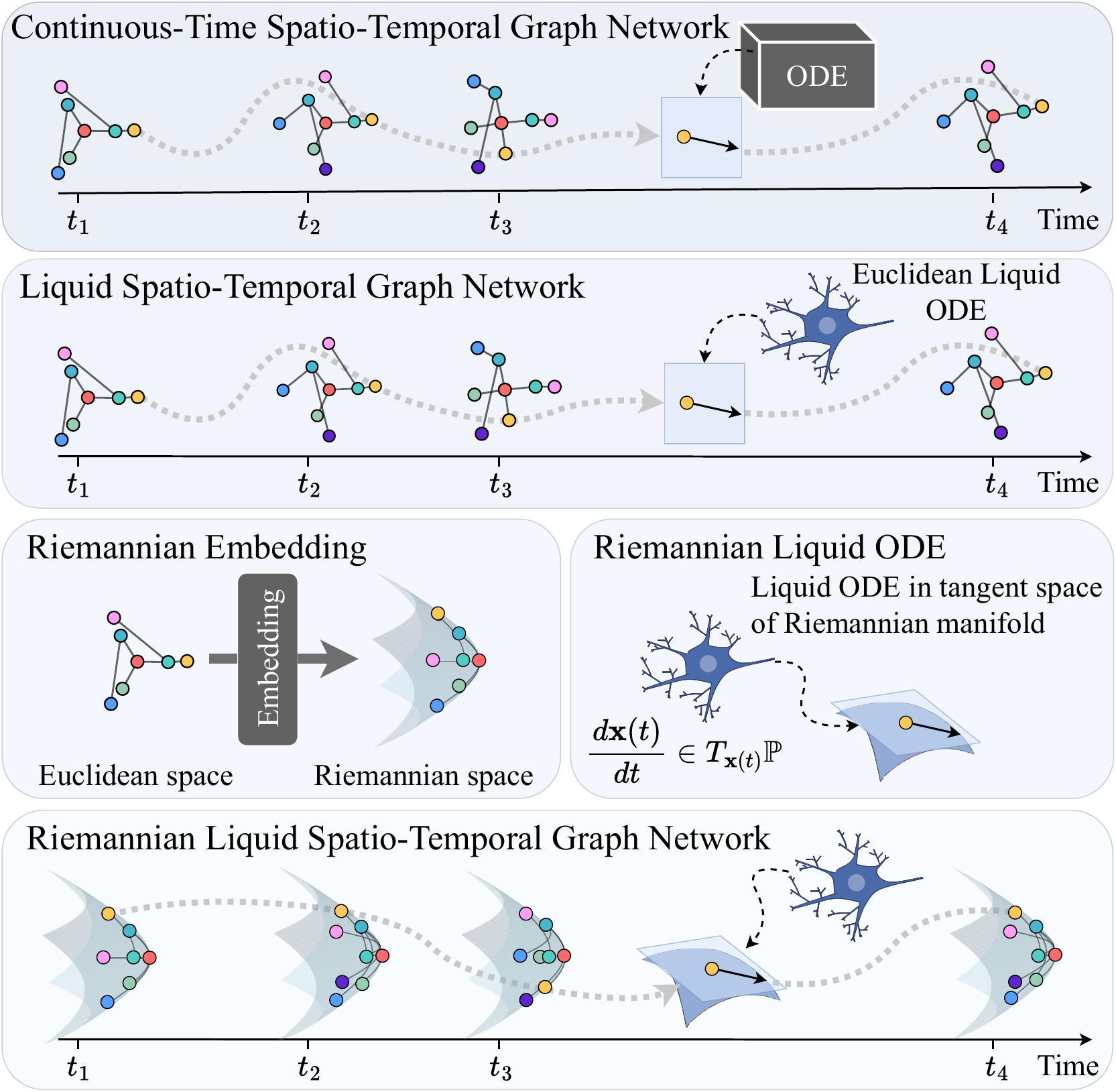}
    % \vspace{-pt}
    \caption{Euclidean continuous-time models, including Liquid networks, distort non-Euclidean graph structures. RLSTG reduces distortion by embedding graphs onto suitable Riemannian manifolds. The system's state evolves on this curved space via a Liquid ODE formulated in the manifold's tangent spaces, enabling a more faithful representation.}
    \label{fig:intro}
    \vspace{-7pt}
\end{figure}

\section{Introduction}

Many real-world dynamical systems evolve in continuous time, with observations often collected at irregular intervals \cite{tgn_icml_grl2020}. For instance, user interactions in social networks occur at sporadic time points, and sensor failures in traffic monitoring can lead to sparse data \cite{gravina2024tgode}. Continuous-time models (e.g. CT-RNNs \cite{funahashi1993approximation}, Neural ODE \cite{chen2018neuralode}) are inherently well-suited for handling such irregularly and sparsely sampled time-series data \cite{rubanova2019latent, kidger2020neural}. Liquid Time-Constant networks (LTCs), an advanced class of continuous-time neural ODE models \cite{hasani2021liquid, hasani2022closed} that have been successfully extended to graphs \cite{marino2024liquid}. Owing to their adaptive gating mechanisms and learnable time constants, LTCs have shown exceptional stability and expressivity in processing such data.

Despite their success in capturing complex temporal patterns, existing continuous-time graph models share a fundamental limitation: they operate exclusively within Euclidean space \cite{tgn_icml_grl2020, gravina2024tgode}. This assumption is problematic, as many real-world graphs are inherently non-Euclidean \cite{nickel2017poincare}. For example, social networks and citation graphs often exhibit tree-like hierarchies, while transportation networks and molecular graphs contain numerous cycles \cite{chami2019hyperbolic, mostafa2022hyperbolic}. Embedding these structures into the flat geometry of Euclidean space inevitably causes significant distortion, compressing hierarchies and stretching cycles, which degrades the quality of the learned representations and harms the performance of downstream tasks.

To address this geometric distortion, the field of geometric deep learning has proposed modeling graph data on Riemannian manifolds, which offer a richer geometric representation \cite{ganea2018hyperbolic}. Hyperbolic manifolds, with their negative curvature, provide an ideal space for embedding hierarchical structures with minimal distortion, while spherical manifolds are naturally suited for capturing cyclical patterns \cite{nickel2017poincare, chami2019hyperbolic}. However, existing Riemannian Graph Neural Networks (GNNs) have predominantly focused on static graphs or employed discrete-time evolution mechanisms \cite{yang2021discrete, bai2023hgwavenet}, the challenge of continuous-time dynamics on manifolds has remained largely unexplored.

To bridge this gap, we introduce the Riemannian Liquid Spatio-Temporal Graph Network (RLSTG), a model to unify LTCs with Riemannian geometry for spatio-temporal graphs. We design RLSTG on the tangent spaces of the manifold, using mappings defined within the context of Riemannian geometry. As conceptualized in Figure~\ref{fig:intro}, our approach is uniquely equipped to handle the dual challenges of irregular temporal sampling and complex spatial geometry.

Our main contributions are summarized as follows:
\begin{itemize}
    \item We develop a specialized, provably convergent ODE solver for stiff dynamics on manifolds, extend the theoretical guarantees of stability and universal approximation for LTCs to the Riemannian domain and provide expressivity analyses.
    \item We propose RLSTG, a framework integrates liquid continuous-time dynamics with Riemannian geometry for spatio-temporal graph representation learning.
    \item We conduct extensive experiments on real-world benchmarks, demonstrating that RLSTG is significantly superior in irregularly sampled data and graphs with non-Euclidean structures.
\end{itemize}

\section{Related Work}

Our work integrates three research pillars: continuous-time neural models, spatio-temporal graph networks, and geometric deep learning.

\subsubsection{Continuous-Time and Spatio-Temporal Models}
Continuous-time model \cite{funahashi1993approximation, chen2018neuralode, rubanova2019latent, kidger2020neural}, especially Liquid Time-Constant networks \cite{hasani2021liquid, hasaniliquid} established a foundation for modeling continuous, irregularly sampled dynamics.  
However, these models and most continuous-time graph networks are fundamentally Euclidean \cite{fang2021spatial, gravina2024tgode, gravina2024ctan, wan2025rethink}.  This limitation is also shared by dominant discrete-time models, which are further constrained by fixed time intervals ~\cite{li2018dcrnn_traffic, guo2019attention}. RLSTG overcomes these dual constraints by formulating a stable, liquid continuous-time model that evolves directly on a geometrically appropriate manifold.

\subsubsection{Geometric Graph Learning}
Geometric deep learning mitigates the metric distortion of embedding non-Euclidean graphs into flat space. Foundational work on non-Euclidean GNNs \cite{nickel2017poincare, ganea2018hyperbolic, chami2019hyperbolic, gu2018learning, bachmann2020constant} demonstrated superior performance for static non-Euclidean data. Recently, these geometric concepts have been extended to dynamic graphs \cite{bai2023hgwavenet, yang2021discrete}, leveraging hyperbolic space for temporal link prediction and long-term dynamics. Nevertheless, these advanced geometric approaches still predominantly rely on discrete-time update mechanisms. RLSTG bridges this gap by unifying LTCs with Riemannian geometry, enabling a continuous-time dynamical system on a curved space.

\section{Preliminaries}

\subsubsection{Riemannian Geometry Basics.}
Riemannian geometry provides a rigorous framework for analyzing smooth manifolds endowed with smoothly varying inner products on their tangent spaces. A Riemannian manifold \((\mathcal{M}, g)\) consists of a smooth manifold \(\mathcal{M}\) together with an inner product \(g_{\mathbf{x}}(\cdot,\cdot)\) on each tangent space \(T_{\mathbf{x}}\mathcal{M}\) that depends smoothly on \(\mathbf{x}\); since every \(T_{\mathbf{x}}\mathcal{M}\) is an inner-product space, one can define angles, lengths, and geodesics, where a geodesic is the shortest smooth curve joining two points \(\mathbf{x},\mathbf{y}\in\mathcal{M}\). Canonical examples include constant-curvature spaces such as hyperbolic space \(\mathbb{H}^{n}\), spherical space \(\mathbb{S}^{n}\), and Euclidean space \(\mathbb{E}^{n}\) (detailed in Appendix~\ref{Basic_Geo}). More complex geometries arise as Cartesian products of simpler factors, yielding a product manifold that supports heterogeneous structure (for example, \(\mathbb{H}^{n}\times\mathbb{S}^{m}\)). Given smooth manifolds \(\mathcal{M}_{1},\ldots,\mathcal{M}_{k}\), the product manifold is
\begin{equation}
\mathbb{P}=\mathcal{M}_{1}\times\mathcal{M}_{2}\times\cdots\times\mathcal{M}_{k}.
\end{equation}
A point \(\mathbf{x}\in\mathbb{P}\) is written \(\mathbf{x}=[x_{1},\ldots,x_{k}]\) with \(x_{i}\in\mathcal{M}_{i}\), and the tangent space at \(\mathbf{x}\) splits as a direct sum, reflecting the component-wise differential structure:
\begin{equation}
T_{\mathbf{x}}\mathbb{P}\cong T_{x_{1}}\mathcal{M}_{1}\oplus\cdots\oplus T_{x_{k}}\mathcal{M}_{k}.
\end{equation}
Accordingly, any tangent vector \(\mathbf{v}\in T_{\mathbf{x}}\mathbb{P}\) decomposes as \(\mathbf{v}=[v_{1},\ldots,v_{k}]\) with \(v_{i}\in T_{x_{i}}\mathcal{M}_{i}\). If each factor \(\mathcal{M}_{i}\) carries a distinguished origin \(o_{i}\in\mathcal{M}_{i}\), the natural origin of the product is
\begin{equation}
\mathbf{o}=[o_{1},o_{2},\ldots,o_{k}]\in\mathbb{P}.
\end{equation}
For product spaces we adopt a decoupled geodesic distance that aggregates the factor-wise distances,
\begin{equation}
d_{\mathbb{P}}(\mathbf{x},\mathbf{y})=\sum_{i=1}^{k} d_{\mathcal{M}_{i}}(x_{i},y_{i}),
\end{equation}
where \(d_{\mathcal{M}_{i}}\) denotes the geodesic distance on \(\mathcal{M}_{i}\). Two fundamental operators mediate between points on a manifold and vectors in its tangent spaces: the exponential map \(\exp_{\mathbf{x}}:T_{\mathbf{x}}\mathcal{M}\to\mathcal{M}\) and the logarithmic map \(\log_{\mathbf{x}}:\mathcal{M}\to T_{\mathbf{x}}\mathcal{M}\). On \(\mathbb{P}\) these operators act component-wise, implementing manifold-aware moves by independent factor updates:
\begin{equation}
\mathrm{Exp}_{\mathbf{x}}(\mathbf{v})=\bigl[\exp^{K_{1}}_{x_{1}}(v_{1}),\;\exp^{K_{2}}_{x_{2}}(v_{2}),\;\ldots,\;\exp^{K_{k}}_{x_{k}}(v_{k})\bigr],
\end{equation}
\begin{equation}
\mathrm{Log}_{\mathbf{x}}(\mathbf{y})=\bigl[\log^{K_{1}}_{x_{1}}(y_{1}),\;\log^{K_{2}}_{x_{2}}(y_{2}),\;\ldots,\;\log^{K_{k}}_{x_{k}}(y_{k})\bigr].
\end{equation}
The orthogonal projection onto the product tangent space decomposes factorwise:
\begin{equation}
\mathrm{Proj}_{T_{\mathbf{x}}\mathbb{P}}(\mathbf{u})
=
\bigl[\mathrm{Proj}_{T_{x_{1}}\mathcal{M}_{1}}(u_{1}),\mathrm{Proj}_{T_{x_{2}}\mathcal{M}_{2}}(u_{2}),\ldots,\mathrm{Proj}_{T_{x_{k}}\mathcal{M}_{k}}(u_{k})\bigr].
\end{equation}
Parallel transport also decomposes by factors: if \(\gamma(t)=[\gamma_{1}(t),\ldots,\) \(\gamma_{k}(t)]\) is a smooth curve in \(\mathbb{P}\) and \(\mathbf{v}(t)=[v_{1}(t),\ldots,v_{k}(t)]\) is parallel along \(\gamma(t)\), then each \(v_{i}(t)\) is parallel along \(\gamma_{i}(t)\) in \(\mathcal{M}_{i}\); denoting the parallel transport in \(\mathcal{M}_{i}\) from \(\gamma_{i}(0)\) to \(\gamma_{i}(1)\) by \(P^{\mathcal{M}_{i}}_{\gamma_{i}(0)\to\gamma_{i}(1)}\), the product transport is given by
\begin{equation}
P^{\mathbb{P}}_{\gamma(0)\to\gamma(1)}\,[v_{1},\ldots,v_{k}]
=
\bigl[P^{\mathcal{M}_{1}}_{\gamma_{1}(0)\to\gamma_{1}(1)} v_{1},\;\ldots,\;P^{\mathcal{M}_{k}}_{\gamma_{k}(0)\to\gamma_{k}(1)} v_{k}\bigr],
\end{equation}
and these operators together enable seamless transitions between points and tangent directions, which are essential for defining and integrating Riemannian ordinary differential equations on product geometries.

\subsubsection{Graph and Aggregation on Manifolds.}
We model a continuous-time spatio-temporal graph as $G(t) = (\mathcal{V}(t), \mathcal{E}(t), \mathbf{N}(t), \mathbf{E}(t))$, where $\mathcal{V}(t)=\{v_1(t),\ldots,v_{|\mathcal{V}|}(t)\}$ is the set of vertices, $\mathcal{E}(t) \subseteq \mathcal{V}(t) \times \mathcal{V}(t)$ is the set of edges, $\mathbf{N}(t) \in \mathbb{R}^{|\mathcal{V}| \times d_v}$ is the node feature matrix, and $\mathbf{E}(t) \in \mathbb{R}^{|\mathcal{E}| \times d_e}$ is the edge feature matrix. To capture the data's intrinsic geometry, we embed these features onto a Riemannian manifold $\mathbb{P}$. Each node $v_i$ is represented by a learnable point $\mathbf{x}_i \in \mathbb{P}$, and each edge $(i,j)$ by a point $\mathbf{a}_{ij} \in \mathbb{P}$, mapped from their initial features via the exponential map:
\begin{equation}\label{eq:node_encoder}
    \mathbf{x}_i = \phi_v(\mathbf{n}_i) = \mathrm{Exp}_{\mathbf{o}}(W_v \mathbf{n}_i + \mathbf{b}_v),
\end{equation}
\begin{equation}\label{eq:edge_encoder}
    \mathbf{a}_{ij} = \psi_e(\mathbf{e}_{ij}) = \mathrm{Exp}_{\mathbf{o}}(W_e \mathbf{e}_{ij} + \mathbf{b}_e).
\end{equation}

Node representations are updated via a message-passing scheme adapted for Riemannian manifolds. The process is conducted in the tangent space at the origin $\mathbf{o}$ for computational efficiency and geometric consistency. For each neighboring node $j \in \mathcal{N}(i)$, a message vector $\mathbf{m}_{ij}$ is computed. This message is a function of the source node, target node, and their connecting edge, all projected into the tangent space:
\begin{equation}\label{eq:message_creation}
\mathbf{m}_{ij} = \sigma\Big(\mathrm{MLP}\Big(\big[\mathrm{Log}_{\mathbf{o}}(\mathbf{x}_{i}), \mathrm{Log}_{\mathbf{o}}(\mathbf{x}_{j}),\mathrm{Log}_{\mathbf{o}}(\mathbf{a}_{ij})\big]\Big)\Big),
\end{equation}
where $\sigma$ is a non-linear activation function (e.g. tanh). These message vectors are aggregated via summation. The resulting vector, which may not lie in the tangent space, is projected back to ensure geometric validity. Finally, this projected vector is mapped from the tangent space to the manifold to produce the updated representation for node $i$:
\begin{equation}\label{eq:aggregation}
\mathrm{AGG}(\mathbf{x})_i = \mathrm{Exp}_{\mathbf{o}}\left( \mathrm{Proj}_{T_{\mathbf{o}}\mathbb{P}}\left(\sum_{j \in \mathcal{N}(i)} \mathbf{m}_{ij} \right)\right).
\end{equation}

This message-passing mechanism allows for a rich, geometrically sound fusion of neighborhood information, forming the basis for our Riemannian graph differential equation.

\subsubsection{Liquid Time-Constant Networks.}
\textit{Liquid time-constant networks} (LTCs) are a class of continuous-time RNNs with superior expressivity and stability \cite{hasani2021liquid}. Instead of an unstructured derivative, LTCs use an ODE whose parameters are modulated by an adaptive gating mechanism. The dynamics of a hidden state $\mathbf{x}^{E}(t)\in \mathbb{E}^n$ are governed by:
\begin{equation}\label{eq:ltc}
\frac{d\mathbf{x}^{E}(t)}{dt} = -\frac{1}{\boldsymbol{\widetilde{\tau}}_\text{sys}} \odot \mathbf{x}^{E}(t) + \widetilde{\mathbf{f}}(\mathbf{x}^{E}(t),\mathbf{I}(t),t,\theta) \odot \mathbf{V}^E,
\end{equation}
where the driving input $\mathbf{V}$ is learnable and the key innovation is the \textit{liquid time-constant} $\boldsymbol{\widetilde{\tau}}_\text{sys}$, which changes dynamically based on the current state and input via a gating function $\widetilde{\mathbf{f}}(\cdot)$:
\begin{equation}
\boldsymbol{\widetilde{\tau}}_\text{sys} =  \frac{\boldsymbol{\tau}}{1+\boldsymbol{\tau}\odot \widetilde{\mathbf{f}}(\mathbf{x}^{E}(t),\mathbf{I}(t),t,\theta)}.
\end{equation}

The operator $\odot$ denotes the Hadamard product (element-wise multiplication), and all divisions are also performed element-wise. This state-dependent modulation allows LTCs to learn flexible temporal patterns while maintaining provably stable, bounded hidden states. However, this formulation is fundamentally Euclidean; operations like vector addition and scaling are ill-defined on curved manifolds. This necessitates a new formulation of liquid dynamics that respects Riemannian geometry, which we introduce next.

\section{Riemannian Liquid Spatio-Temporal Graph Network}
RLSTG is a continuous-time graph neural network that evolves node features on a Riemannian manifold via a liquid time-constant ODE. We describe the model’s state representation, ODE formulation, and ODE solver. We summarize key theoretical properties: the stability of the system and the advantage on expressive power.

\subsection{Riemannian Liquid ODE Formulation}
We generalize the principles of Liquid Time-Constant networks (LTCs) to evolve node states on a Riemannian manifold $\mathbb{P}$. The standard LTCs governs a state $\mathbf{x}^E$ in Euclidean space $\mathbb{E}$ via a leak term $(-\mathbf{x}^E)$ and a driving input $(\mathbf{V})$. We establish a direct analogy on the manifold. The state's decay towards the origin $\mathbf{o}$ is captured by the logarithmic map $\mathrm{Log}_{\mathbf{x}}(\mathbf{o}) \in T_{\mathbf{x}}\mathbb{P}$, a tangent vector at the current state $\mathbf{x}$ pointing towards $\mathbf{o}$. The driving input, represented as a vector $\mathbf{V} \in T_{\mathbf{o}}\mathbb{P}$, is translated to the state's local frame via parallel transport, yielding $P_{\mathbf{o}\to\mathbf{x}}(\mathbf{V}) \in T_{\mathbf{x}}\mathbb{P}$. 
% Anchoring the dynamics to a fixed origin $\mathbf{o}$ is a key design choice for achieving provable theoretical stability and computational efficiency. Its potential expressive limitations are mitigated by the adaptive gating function and the space's isotropy.

These substitutions yield the Riemannian Liquid Spatio-Temporal Graph (RLSTG) ODE. The dynamics for a node state $\mathbf{x}(t) \in \mathbb{P}$ are defined entirely within its tangent space $T_{\mathbf{x}(t)}\mathbb{P}$:
\begin{equation}\label{eq:rlstg_ode}
\begin{aligned}
\frac{d \mathbf{x}(t)}{dt}=
\mathrm{Proj}_{T_{\mathbf{x}(t)}\mathbb{P}}
& \left( \left( \frac{1}{\boldsymbol{\tau}}+f(\cdot) \right) \odot\mathrm{Log}_{\mathbf{x}(t)}\mathbf{o}+ \right.\left.f(\cdot)\odot P_{\mathbf{o}\rightarrow \mathbf{x}(t)}(\mathbf{V}) \right).
\end{aligned}
\end{equation}

The dynamics are modulated by a state-dependent gating function $f(\cdot)$, which is computed for each node $i$ as:
\begin{equation}
\label{eq:gating_f}
\begin{aligned}
f(\mathbf{x}(t),\mathbf{N}(t),\mathbf{E}(t),t,\theta)_i
=&\sigma(
 W_{\mathbf{x}}\mathrm{Log}_{\mathbf{o}}(\mathrm{AGG}(\mathbf{x}(t))_i)+\\
 &W_{\mathbf{I}}\phi_v(\mathbf{N}(t)_i)
+b),
\end{aligned}
\end{equation}
where $\mathrm{AGG}(\cdot)$ performs geometric aggregation of neighboring node states as defined in Eq.~\eqref{eq:aggregation}, $\phi_v(\cdot)$ is an encoder defined in Eq.~\eqref{eq:node_encoder} and $\sigma$ is a non-linear activation function. The function $f(\cdot)$ adaptively controls the system's behavior: it simultaneously modulates the rate of decay towards the stable equilibrium point $\mathbf{o}$ and scales the influence of a learnable driving input vector $\mathbf{V}$. This dual role allows the system to learn complex temporal dynamics, balancing stability with responsiveness to new inputs.

\subsection{Geodesic Decay ODE Solver}
Numerically integrating Eq.~\eqref{eq:rlstg_ode} is challenging due to stiffness caused by the adaptive decay term~\cite{hairer2006structure}. Standard explicit solvers require prohibitively small step sizes for stability, while implicit methods are too computationally expensive on manifolds.

To overcome this, we introduce the \textit{Geodesic Decay (GD) ODE solver}, a specialized integrator based on operator splitting~\cite{stern2009implicit, leimkuhler2004simulating}. This technique decomposes the ODE's vector field into its non-stiff (driving) and stiff (decay) components. We handle the stiff decay analytically for stability and integrate the non-stiff driving term efficiently, yielding a robust solution that avoids the limitations of purely explicit or implicit methods. 

The solver update, illustrated in Figure~\ref{fig:gd}, is a two-stage process. The driving step advances the state to an intermediate point $\mathbf{x}^{*}$:

\begin{equation}
\mathbf{x}^{*} = \mathrm{Exp}_{\mathbf{x}(t)}\Big(\Delta t \cdot \left(f(\mathbf{x}(t),\mathbf{N}(t),\mathbf{E}(t),t,\theta) \odot P_{\mathbf{o}\rightarrow \mathbf{x}_t}(\mathbf{V})\right)\Big).
\end{equation}

The GD step pulls $\mathbf{x}^{*}$ towards the origin using a state-dependent decay factor $\alpha$ to yield the final state:

\begin{equation}
    \alpha = e^{-\Delta t \cdot \left(\frac{1}{\boldsymbol{\tau}} + f(\mathbf{x}(t),\mathbf{N}(t),\mathbf{E}(t),t,\theta)\right)},
\end{equation}

\begin{equation}\label{eq:gd_solver}
\mathbf{x}(t+\Delta t) = \mathrm{Exp}_{\mathbf{o}} \Big(\alpha \odot 
\mathrm{Log}_{\mathbf{o}}( \mathbf{x}^{*}) \Big).
\end{equation}

The complete forward propagation for a discrete time step $\Delta t_k$ from state $\mathbf{x}_k$ is thus:
\begin{equation}\label{eq:gd_solver_d}
\begin{aligned}
\mathbf{x}_{k+1}
={}&
\operatorname{Exp}_{\mathbf{o}}\Bigl(
  e^{-\Delta t_k\bigl[\frac{1}{\boldsymbol{\tau}}+f(\cdot)\bigr]}
  \;\odot\\
&\operatorname{Log}_{\mathbf{o}}\Bigl(
    \operatorname{Exp}_{\mathbf{x}_k}\bigl(
        \Delta t_k\,(\,f(\cdot)\odot P_{\mathbf{o}\to\mathbf{x}_k}(\mathbf{V})\,)
    \bigr)
  \Bigr)
\Bigr),
\end{aligned}
\end{equation}
where $f(\cdot)=f(\mathbf{x}(t),\mathbf{N}(t),\mathbf{E}(t),t,\theta)$.

\begin{theorem}
Let the manifold be a product manifold $\mathbb{P}$. Consider the ordinary differential equation on $\mathbb{P}$ given by Eq.~\eqref{eq:rlstg_ode}. The GD solver defined by Eq.~\eqref{eq:gd_solver_d} is a first-order convergent method, where the global error $E_{\text{global}}$ over a finite time interval $[0, T]$ is bounded by:
\begin{equation}
E_{\text{global}} \approx O(\Delta t).
\end{equation}
\end{theorem}

\begin{figure}[t]
    \centering
    \includegraphics[width=0.7\linewidth]{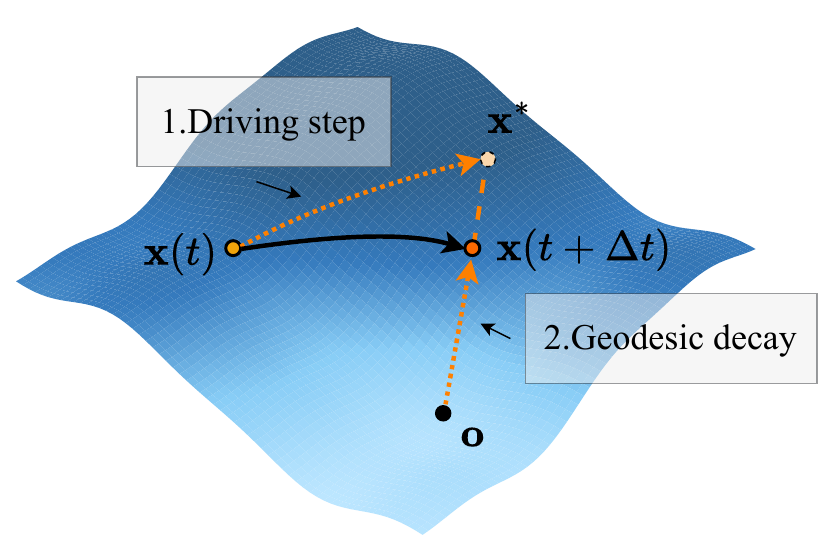}
    \vspace{-10pt}
    \caption{GD ODE solver employs an operator splitting technique to decouple the stiff dynamics. \textbf{(1) Driving step.} An explicit Euler update advances the state from $\mathbf{x}(t)$ to an intermediate point $\mathbf{x}^*$ based on the non-stiff driving input. \textbf{(2) GD:} The intermediate state $\mathbf{x}^*$ is analytically pulled towards the origin $\mathbf{o}$ to yield the final state $\mathbf{x}(t+\Delta t)$.}
    \label{fig:gd}
    \vspace{-10pt}
\end{figure}

As proven in Appendix~\ref{pf:solver_error}, this solver is first-order convergent.
The GD ODE solver therefore provides a robust and efficient framework for integrating our stiff dynamical system on the manifold, making it exceptionally well-suited for deep learning architectures.

\subsection{Numerical Stability of RLSTG}

Establishing the numerical stability of RLSTG is critical for its reliability. Riemannian embeddings, particularly in hyperbolic space, are commonly prone to numerical instability~\cite{chami2019hyperbolic}. Therefore, providing a rigorous theoretical guarantee for stability is not merely a formality but a practical necessity. This requires generalizing the stability analysis of Euclidean LTCs~\cite{hasani2021liquid} to the non-trivial Riemannian setting, where curved geometry fundamentally alters the system dynamics.

\begin{theorem}
Let the state of a node, $\mathbf{x}(t)$, evolve on a product Riemannian manifold $\mathbb{P} = \prod_{i=1}^{k} \mathcal{M}_i$ with origin $\mathbf{o} = (\mathbf{o}_1, \dots, \mathbf{o}_k)$ according to Eq.~\eqref{eq:rlstg_ode}. We define the system time constant for each component as $(\boldsymbol{\tau}_{\text{sys}})_i = \frac{\boldsymbol{\tau}_i}{1 + \boldsymbol{\tau}_i f_i(\cdot)}$. Then, for each component $i$ of the system, this quantity is bounded as follows:
\begin{equation}
\frac{\boldsymbol{\tau}_i}{1+\boldsymbol{\tau}_i} \le (\boldsymbol{\tau}_{\text{sys}})_i \le \tau_i.
\end{equation}
\end{theorem}

As proven in Appendix \ref{pf:tau}, this theorem provides a foundational guarantee for stability. By leveraging the $[0, 1]$ range of the sigmoid gating function $f_i(\cdot)$, we show that the adaptive time constant governing the system's decay rate is strictly bounded. This prevents vanishing or exploding decay rates, a crucial step towards ensuring the overall boundedness of the state trajectory.

With the system time constant bounded, we establish our main stability result: the state trajectory itself is confined to a fixed region on the manifold.

\begin{theorem}
On the product manifold $\mathbb{P} = \prod_{i=1}^{k} \mathcal{M}_i$, denote the norm on $\mathcal{M}_i$ as $\|\cdot\|_i$. Let the state $\mathbf{x}(t) \in \mathbb{P}$ evolve according to the Riemannian Liquid ODE. Assume that the condition $d_i(\mathrm{exp}_{\mathbf{o}_i}(\pm \mathbf{V}_i), \mathbf{o}_i) \ge \boldsymbol{\tau}_i \| \mathbf{V}_i\|_i$ is met for each component. We denote $R = d(\mathrm{Exp}_{\mathbf{o}}( \mathbf{V}),\mathbf{o})$ and $r = d(\mathrm{Exp}_{\mathbf{o}}(-\mathbf{V}),\mathbf{o})$. Then for any finite time interval $[0,T]$, the state $\mathbf{x}(t)$ is bounded within an annulus centered at the origin $\mathbf{o}$:
\begin{equation}
r \le d(\mathbf{x}(t),\mathbf{o}) \le R.
\end{equation}
\end{theorem}

The proof in Appendix \ref{pf:state} analyzes the time evolution of the state's geodesic distance to the origin, $d(\mathbf{x}(t), \mathbf{o})$. We show that the ODE's vector field creates a potential well on the manifold, ensuring any trajectory is perpetually confined within the boundary.

\subsection{Expressive power of RLSTG}

\subsubsection{Universal Approximation Theorem.}
We establish the universal approximation capability for RLSTG, extending the known property from Euclidean models~\cite{hornik1989multilayer, funahashi1993approximation, hasani2021liquid, kratsios2022universal-a, kratsios2022universal-b} to our framework on curved Riemannian manifolds, thereby bridging a key theoretical gap.

\begin{theorem}
Assume that the solution $\mathbf{u}(t)$ of an ODE on Riemannian manifold exists and is unique on the time interval $[0, T]$, and the solution trajectory is contained within a compact subset $D$ of an open set $S \subset \mathbb{P}$. Given an arbitrary precision $\varepsilon > 0$ and a time interval length $T > 0$, there exists a RLSTG, with dynamics defined by  Eq.~\eqref{eq:gd_solver_d}, such that the sequence $\{\mathbf{x}_k\}$ generated by the discrete network iterations satisfies:
\begin{equation}
\max_{0 \le k \le N} d(\mathbf{x}_k, \mathbf{u}(t_k)) < \varepsilon.
\end{equation}
\end{theorem}

The complete proof is provided in Appendix~\ref{appendix:UAT}. The core strategy is to project the target vector field onto a common tangent space, where the classical universal approximation theorem can be invoked. We show that RLSTG architecture is capable of realizing this approximating network and bound the total error to prove that our model's trajectory can approximate the target trajectory to any specified precision.

\subsubsection{Trajectory Length as an Expressivity Metric.}

We quantify the model's expressive power via the length of its state trajectory in the tangent space. Let the global state be $X(t)=[x_1(t),\dots,x_n(t)]\in\mathbb P$ and fix an origin $o\in\mathbb P$ so that $\mathrm{Log}_o$ is well-defined on the (compact) domain reached by the trajectories. Write the tangent representation $Z(t)=\mathrm{Log}_o(X(t))=[z_1(t),\dots,z_n(t)]\in T_o\mathbb P\cong\mathbb R^d$ and use the Euclidean norm $\|\cdot\|_o$ on $T_o\mathbb P$. The expressivity metric is the curve length:
\begin{equation}
\ell(Z):=\int_0^T \big\|\tfrac{d}{dt}Z(t)\big\|_o\,dt.
\end{equation}
A longer trajectory implies a greater capacity to capture complex dynamics. Our analysis establishes that the expected trajectory length is lower-bounded by the quality of the spatial graph representation, captured by the expected squared norm of node representations in the tangent space, $\mathbb{E}[||z||^{2}]$:
\begin{equation}
\mathbb{E}[l(Z(t))] \ge \mathcal{O}(C_{\text{arch}} \cdot \mathbb{E}[||z||^{2}]) \cdot l(I(t)),
\end{equation}
where $C_{\text{arch}}$ is an architecture-dependent constant and $l(I(t))$ is the input signal's trajectory length.

This inequality directly links expressivity to metric distortion. Embedding a graph into a space of mismatched curvature (e.g., a tree in Euclidean space) compresses the representation, reducing $\mathbb{E}[||z||^{2}]$ and thus limiting the model's expressive power. Conversely, aligning the manifold's curvature with the graph's intrinsic geometry (e.g., Hyperbolic for trees, Spherical for cycles) minimizes distortion. This yields a more expansive representation with a larger $\mathbb{E}[||z||^{2}]$, thereby enhancing expressivity. The subsequent theorems formalize this geometric advantage for canonical graph structures.

\begin{figure*}[t]
    \centering
    \includegraphics[width=0.9\linewidth]{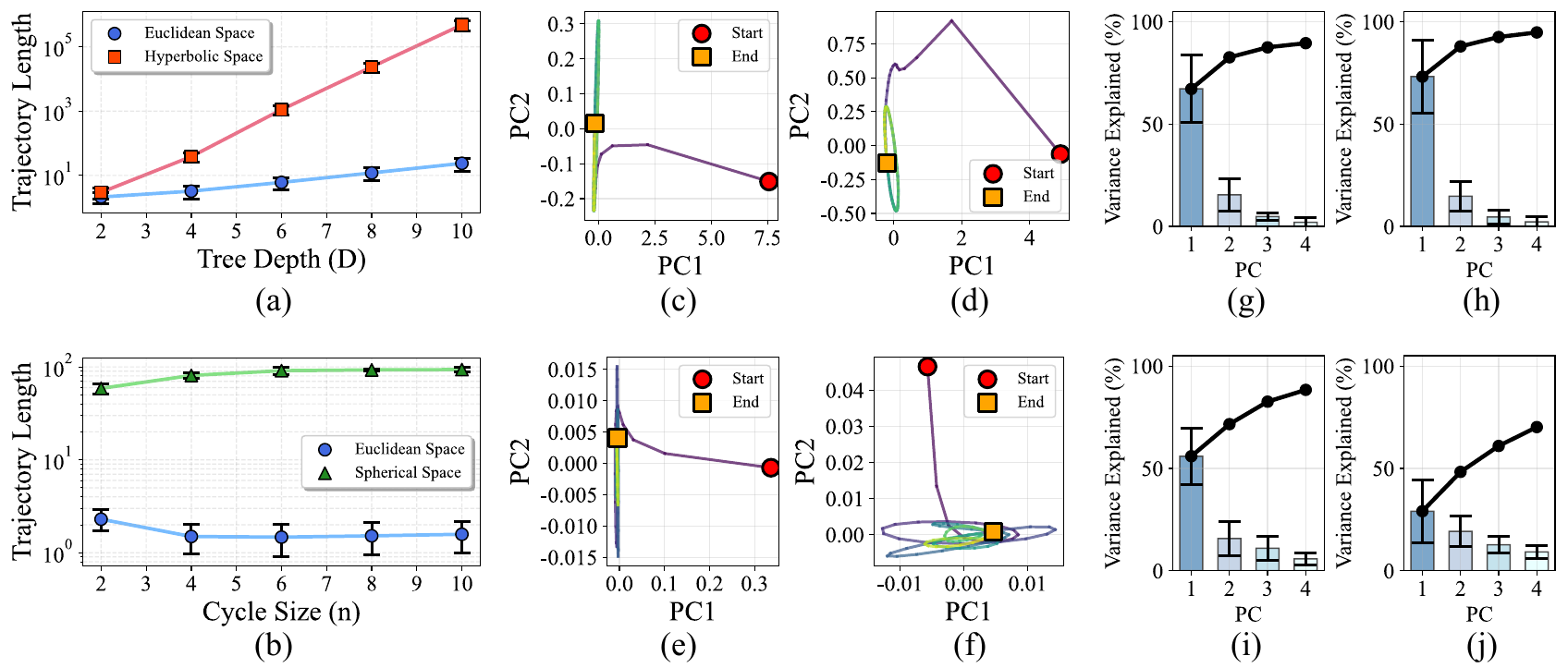}
    \vspace{-10pt}
    \caption{
Expressive power analysis.
        (a) Mean trajectory length versus tree depth (D) in Euclidean and Hyperbolic spaces.
        (b) Mean trajectory length versus cycle size (n) in Euclidean and Spherical spaces.
        (c, d) 2D Principal Component Analysis (PCA, described in~\cite{hasani2021liquid}) projections of the system trajectory for a tree graph (D=10) in Euclidean and Hyperbolic space.
        (e, f) 2D PCA projections for a cycle graph (n=10) in Euclidean and Spherical space.
        (g-j) Variance explained by the first four principal components (bars) and their cumulative contribution (black line). Each plot corresponds to an experimental setting: (g) Tree in Euclidean, (h) Tree in Hyperbolic, (i) Cycle in Euclidean, and (j) Cycle in Spherical space.
        In (c-f), the color gradient on the trajectory line indicates the temporal progression from the start to the end state.
    }
    \label{fig:traj}
\end{figure*}

\begin{theorem}
For an RLSTG model processing a tree of depth $D$, the lower bounds on the expected trajectory length $\mathbb{E}[l(Z(t))]$ for embeddings in Euclidean space $\mathbb{E}$ and Hyperbolic space $\mathbb{H}$ are, respectively:
\begin{equation}
\mathbb{E}[l(Z_{E}(t))] \ge \mathcal{O}(C_{\text{arch}} \cdot D) \cdot l(I(t)),
\end{equation}
\begin{equation}
\mathbb{E}[l(Z_{H}(t))] \ge \mathcal{O}(C_{\text{arch}} \cdot D^{2}) \cdot l(I(t)).
\end{equation}
\end{theorem}

\begin{proof}
By the architectural trajectory–length lower bound established earlier,
\begin{equation}
\label{eq:traj-lb}
\begin{aligned}
\mathbb{E}\!\big[\ell(Z(t))\big]
\;\ge\;
\mathcal{O}\!\Big(C_{\text{arch}}\cdot \mathbb{E}\big[\|z\|^{2}\big]\Big)\,\ell(I(t)).
\end{aligned}
\end{equation}
It suffices to bound $\mathbb{E}\big[\|z_j\|^{2}\big]$ in the two ambient geometries.

For the Euclidean case, embed the tree $T$ with $n$ nodes into $\mathbb{R}^{n}$ by fixing an orthonormal basis $\{e_v\}_{v\in V}$, setting $f(o)=\vec{0}$ for the root $o$, and for a non-root $v$ with parent $p$ defining $f(v)=f(p)+e_v$. If $d_G(o,x_j)=k$ and the unique path is $(v_0=o,v_1,\dots,v_k=x_j)$, then
\begin{equation}
f(x_j)=\sum_{i=1}^k e_{v_i},
\end{equation}
\begin{equation}
d_E(o,x_j)^2
=\Big\|\sum_{i=1}^k e_{v_i}\Big\|^2
=\sum_{i=1}^k \|e_{v_i}\|^2
=k
=d_G(o,x_j).
\end{equation}
Hence for a tree of depth $D$,
\begin{equation}
\begin{aligned}
\mathbb{E}\!\left[\|z_j\|_{E}^{2}\right]
=\mathbb{E}\!\left[d_E(o,x_j)^2\right]
=\mathbb{E}\!\left[d_G(o,j)\right]
=\mathcal{O}(D),
\end{aligned}
\end{equation}
and inserting into \eqref{eq:traj-lb} yields
\begin{equation}
\begin{aligned}
\mathbb{E}\!\left[\ell(Z_E(t))\right]
\;\ge\;
\mathcal{O}\!\big(C_{\text{arch}}\cdot D\big)\,\ell(I(t)).
\end{aligned}
\end{equation}

For the hyperbolic case, it is known that any finite tree admits an embedding into the hyperbolic plane with arbitrarily small distortion~\cite{sarkar2011low}, which gives $\mathrm{dist}(f)\le 1+\varepsilon$. Consequently,
\begin{equation}
\begin{aligned}
d_{\mathbb{H}}(o,x_j)
=\Theta\!\big(d_G(o,j)\big),
\qquad
\|z_j\|_{H}=d_{\mathbb{H}}(o,x_j)
=\Theta\!\big(d_G(o,j)\big),
\end{aligned}
\end{equation}
and for depth $D$,
\begin{equation}
\begin{aligned}
\mathbb{E}\!\left[\|z_j\|_{H}^{2}\right]
=\Theta\!\left(\mathbb{E}\!\left[d_G(o,j)^2\right]\right)
=\Theta(D^2).
\end{aligned}
\end{equation}
Plugging this into \eqref{eq:traj-lb} gives
\begin{equation}
\begin{aligned}
\mathbb{E}\!\left[\ell(Z_H(t))\right]
\;\ge\;
\mathcal{O}\!\big(C_{\text{arch}}\cdot D^{2}\big)\,\ell(I(t)).
\end{aligned}
\end{equation}
\end{proof}

\begin{theorem}\label{thm:cycle-constants}
Let $C_n$ be the $n$-cycle and let $Z_E(t)$ and $Z_S(t)$ denote the RLSTG trajectories in Euclidean and Spherical latent geometries. Then
\begin{equation}
\mathbb{E}\!\big[\ell(Z_E(t))\big] \;\ge\; \frac{C_{\mathrm{arch}}}{4\pi^2}\,n^2\,\ell(I(t))\,.
\end{equation}
Moreover,
\begin{equation}
\mathbb{E}\!\big[\ell(Z_S(t))\big] \;\ge\;
\begin{cases}
\displaystyle \frac{C_{\mathrm{arch}}}{12}\,(n^2-1)\,\ell(I(t))\,, & n \text{ odd},\\[6pt]
\displaystyle \frac{C_{\mathrm{arch}}}{12}\,(n^2+2)\,\ell(I(t))\,, & n \text{ even},
\end{cases}
\end{equation}
and therefore, for all $n\ge 3$,
\begin{equation}
\frac{\mathbb{E}[\ell(Z_S(t))]}{\mathbb{E}[\ell(Z_E(t))]}
\;\ge\;
\begin{cases}
\displaystyle \frac{\pi^2}{3}\!\left(1-\frac{1}{n^2}\right)\!, & n \text{ odd},\\[8pt]
\displaystyle \frac{\pi^2}{3}\!\left(1+\frac{2}{n^2}\right)\!, & n \text{ even},
\end{cases}
\end{equation}
and in particular
\begin{equation}
\lim_{n\to\infty}\frac{\mathbb{E}[\ell(Z_S(t))]}{\mathbb{E}[\ell(Z_E(t))]}=\frac{\pi^2}{3}\,.
\end{equation}
\end{theorem}

\begin{proof}
By the generic lower bound \eqref{eq:traj-lb}, a universal constant is absorbed into $C_{\mathrm{arch}}$$ so that $$\mathbb{E}[\ell(Z(t))]\ge C_{\mathrm{arch}}\,\mathbb{E}[\|z_j\|^2]\,\ell(I(t))$ holds for both geometries.

Euclidean case. Place vertices at$f_E(j)=r_E\big(\cos(2\pi j/n),\,\sin(2\pi j/n)\big)$, $r_E=\frac{1}{2\sin(\pi/n)}$, so adjacent Euclidean distances equal $1$. With the base point at the center, $\|z_j\|_E=r_E$ for all $j$, hence $\mathbb{E}[\|z_j\|_E^2]=r_E^2\;\ge\;\frac{n^2}{4\pi^2}$, since $\sin x\le x \text{ on }[0,\pi/2]$, and thus
\begin{equation}
\mathbb{E}\!\big[\ell(Z_E(t))\big]\;\ge\;\frac{C_{\mathrm{arch}}}{4\pi^2}\,n^2\,\ell(I(t))\,.
\end{equation}

Spherical case. Embed $C_n$ isometrically on $S^1$ of radius $r_S=n/(2\pi)$ so consecutive vertices have geodesic distance $1$~\cite{burago2001course}. Let $m\in\{0,1,\dots,\lfloor n/2\rfloor\}$ be the graph distance from a fixed base point along the shorter arc; then $\|z_j\|_S=m$ and
\begin{equation}
\mathbb{E}[\|z_j\|_S^2]=
\begin{cases}
\displaystyle \frac{2}{2k+1}\sum_{m=1}^{k}m^2=\frac{(2k+1)^2-1}{12}=\frac{n^2-1}{12}, & n=2k+1,\\[6pt]
\displaystyle \frac{2}{2k}\sum_{m=1}^{k-1}m^2+\frac{1}{2k}k^2=\frac{n^2+2}{12}, & n=2k,
\end{cases}
\end{equation}
where
\begin{equation}
\mathbb{E}\!\big[\ell(Z_S(t))\big]\;\ge\;
\begin{cases}
\displaystyle \frac{C_{\mathrm{arch}}}{12}\,(n^2-1)\,\ell(I(t))\,, & n \text{ odd},\\[6pt]
\displaystyle \frac{C_{\mathrm{arch}}}{12}\,(n^2+2)\,\ell(I(t))\,, & n \text{ even}.
\end{cases}
\end{equation}

The ratio bounds follow by division of the two estimates above.
\end{proof}

When geometry aligns with structure, expressivity improves—$D^2$ vs.\ $D$ on trees and, on cycles, $\Theta(n^2)$ in both geometries with a provable constant-factor advantage for Spherical. We empirically validate these findings in Figure \ref{fig:traj}. For a tree graph (c, d), the trajectory in Hyperbolic space is visibly more complex and expansive than its constrained counterpart in Euclidean space. A similar observation holds for a cycle graph (e, f), where the Spherical embedding enables a richer, more intricate trajectory. Finally, the bottom row (g-j) quantifies this complexity by showing the variance explained by the principal components. The higher cumulative variance explained by the first few components in the geometrically-aligned spaces (h, j) compared to the Euclidean ones (g, i) indicates that the learned dynamics are higher-dimensional and more expressive. Together, these results provide a rigorous and visually intuitive confirmation that aligning the model's geometry with the data's structure is fundamental to unlocking its expressive potential.

\section{Experiments}

To validate the RLSTG framework, our empirical evaluation is centered on two canonical tasks, node feature regression and edge link prediction. These tasks are chosen to demonstrate the model's capacity to capture complex spatial geometries and its effectiveness in handling evolving network topologies. The evaluation is completed by ablation studies that quantify the impact of the chosen geometric space and GD ODE solver.

\subsection{RLSTG Adaptation for Downstream Tasks}
RLSTG outputs $x_i\in\mathbb P$. Project to $T_o\mathbb P\cong\mathbb R^d$ and apply Euclidean heads.

\paragraph{Node Regression.}
\begin{align}
h_i = \mathrm{Log}_o(x_i)\in\mathbb R^d, \hat y_i = \mathrm{MLP}_{\mathrm{reg}}(h_i).
\end{align}

\paragraph{Link Prediction.}
\begin{align}
h_i &= \mathrm{Log}_o(x_i),\quad h_j=\mathrm{Log}_o(x_j)\in\mathbb R^d,\\
h_{ij} &= [\,h_i;h_j\,]\in\mathbb R^{2d},\quad
s_{ij} = \mathrm{MLP}_{\mathrm{link}}(h_{ij}),\quad
p_{ij}=\sigma(s_{ij}).
\end{align}

\subsection{Nodes Features Regression}

We evaluate RLSTG on node feature regression using three traffic forecasting benchmarks: METR-LA~\cite{li2018dcrnn_traffic}, PEMS03~\cite{guo2019attention}, and PEMS04 \cite{guo2019attention}. 
The METR-LA dataset is a widely-used real-world dataset for traffic forecasting tasks. It contains traffic speed data collected from 207 sensors on the highways of Los Angeles County over a period of 4 months. The PEMS03 dataset, from the Caltrans Performance Measurement System (PeMS), records 3 months of traffic flow data from 358 sensors in a Californian district. The PEMS04 dataset also originates from the Caltrans Performance Measurement System (PeMS) and covers traffic data from 307 sensors over 2 months in another district.
We compare against several previous models: DCRNN~\cite{li2018dcrnn_traffic}, GMAN~\cite{GMAN-AAAI2020}, AGCRN~\cite{bai2020adaptive}, STGODE~\cite{fang2021spatial}, STG-NCDE~\cite{choi2022STGNCDE}, TGODE~\cite{gravina2024tgode}, LTC~\cite{hasani2021liquid}.

To simulate realistic, non-uniform observation patterns, we subsample the time series, ensuring a minimum temporal distance of 3 between consecutive points. The task is to predict the features of the next sampled time step given the current one. We subsample time indices $\mathcal{T}=\{t_1<\dots<t_M\}$ with a minimum gap of $3$ steps to emulate irregular observations. Data are split train/val/test by time in an $8{:}1{:}1$ ratio. For each $i$, the input is $(N_i,\Delta t_i)$ with $N_i\in \mathbb{R}^{|V|\times d}$ and $\Delta t_i=t_{i+1}-t_i$; the target is $N_{i+1}$. Loss is $\mathcal{L}_{\mathrm{MAE}} \;=\; \frac{1}{|V|\,d}\,\big\lVert \hat{N}_{i+1}-N_{i+1}\big\rVert_{1}.$

The performance comparison, measured by Mean Absolute Error (MAE), is summarized in Table~\ref{tab:traffic_baseline}. The results indicate that our proposed model consistently achieves a lower prediction error compared to a wide range of existing methods across all traffic datasets. This suggests that capturing the intrinsic geometry of the traffic network leads to more accurate forecasts.

\begin{table}[t]
\centering
\small
\caption{MAE $\downarrow$ performance comparison for node features regression on traffic datasets, averaged over 5 separate runs. The best results are in \textbf{bold} and the second best are \underline{underlined}.}
\vspace{-10pt}
\begin{tabular}{l ccc}
\toprule
\textbf{Method} & \textbf{METR-LA} & \textbf{PEMS03} & \textbf{PEMS04} \\
\midrule
DCRNN      & 4.55 $\pm$ 0.14          & 19.53 $\pm$ 0.17          & 26.03 $\pm$ 0.24 \\
GMAN       & 3.78 $\pm$ 0.31          & 18.44 $\pm$ 0.35          & 25.90 $\pm$ 0.29 \\
AGCRN      & 3.58 $\pm$ 0.19          & 18.21 $\pm$ 0.08          & 25.01 $\pm$ 0.17 \\
STGODE     & \underline{2.89 $\pm$ 0.02} & 17.70 $\pm$ 0.05          & 24.58 $\pm$ 0.06 \\
STG-NCDE   & 2.93 $\pm$ 0.04          & 17.72 $\pm$ 0.10          & \underline{24.46 $\pm$ 0.09} \\
TGODE      & 2.98 $\pm$ 0.02          & \underline{17.69 $\pm$ 0.07} & 24.70 $\pm$ 0.04 \\
LTC        & 2.90 $\pm$ 0.03 & 17.74 $\pm$ 0.04          & 24.85 $\pm$ 0.07\\
\midrule
RLSTG      & \textbf{2.77 $\pm$ 0.01} & \textbf{17.49 $\pm$ 0.02} & \textbf{24.08 $\pm$ 0.01} \\
\bottomrule
\end{tabular}
\label{tab:traffic_baseline}
\vspace{-10pt}
\end{table}

\subsection{Edges Links Prediction}

% We conduct experiments on the ENRON social network dataset~\cite{benson2018simplicial} for links prediction task.  The ENRON dataset is a dynamic graph dataset that records email communications among 184 employees over 3 years. We compare RLSTG with previous methods, including JODIE~\cite{kumar2019predicting}, DyRep~\cite{trivedi2019dyrep}, TGAT~\cite{xu2020inductive}, TCL~\cite{wang2021tcl}, TGN~\cite{tgn_icml_grl2020}, GraphMixer~\cite{cong2023do}, DyGFormer~\cite{yu2023towards}, HTGN~\cite{yang2021discrete}, FreeDyG~\cite{tian2024freedyg}, HGWaveNet~\cite{bai2023hgwavenet}. A detailed description of the datasets, which we split into 70\%/15\%/15\% for training, validation, and testing respectively, and the experimental setup is provided in the Appendix~\ref{appendix:Data_Preprocessing}. Table~\ref{tab:link_prediction_hist} shows RLSTG demonstrates advantage in predicting future connections, outperforming other methods in both transductive and inductive scenarios. This highlights its robust generalization capability for dynamic graphs with evolving topologies.

We conduct experiments on the ENRON social network dataset~\cite{benson2018simplicial} for links prediction task.  The ENRON dataset is a dynamic graph dataset that records email communications among 184 employees over 3 years. We compare RLSTG with previous methods, including JODIE~\cite{kumar2019predicting}, DyRep~\cite{trivedi2019dyrep}, TGAT~\cite{xu2020inductive}, TCL~\cite{wang2021tcl}, TGN~\cite{tgn_icml_grl2020}, GraphMixer~\cite{cong2023do}, DyGFormer~\cite{yu2023towards}, HTGN~\cite{yang2021discrete}, FreeDyG~\cite{tian2024freedyg}, HGWaveNet~\cite{bai2023hgwavenet}.For the Enron email network, messages are grouped by day to form snapshots $\{(V_i,E_i)\}_{i=1}^N$. We predict edges at $t_{i+1}$ given the history up to and including $t_i$. The timeline is split $70\%/15\%/15\%$ for train/val/test. Positives are $P_{i+1}=E_{i+1}$. Negatives are drawn with \emph{historical negative sampling}: from $H_i=\bigcup_{k\le i}E_k$, sample an equal-sized set $\tilde N_{i+1}\subseteq H_i\setminus E_{i+1}$ (hard negatives that previously occurred but are absent at $t_{i+1}$). We additionally report performance on \textit{transductive} (both endpoints seen in training) and \textit{inductive} (at least one new node) subsets. The loss on balanced positives/negatives:
$\mathcal{L}_{\mathrm{BCE}} \;=\; -\frac{1}{|\mathcal{D}_{i+1}|}\sum_{(u,v,y)\in \mathcal{D}_{i+1}}
\Big[y\log \hat{p}_{uv}^{(i+1)} + (1-y)\log \big(1-\hat{p}_{uv}^{(i+1)}\big)\Big],$
where $\mathcal{D}_{i+1}=P_{i+1}\cup \tilde N_{i+1}$ and $\hat{p}_{uv}^{(i+1)}$ is the model's predicted probability for edge $(u,v)$ at $t_{i+1}$.
Table~\ref{tab:link_prediction_hist} shows RLSTG demonstrates advantage in predicting future connections, outperforming other methods in both transductive and inductive scenarios. This highlights its robust generalization capability for dynamic graphs with evolving topologies.

\begin{table}[t]
\centering
\small
\caption{Average Precision $\uparrow$ for temporal link-prediction results on ENRON datasets with historical negative sampling strategy, averaged over 5 separate runs. The best results are in \textbf{bold} and the second best are \underline{underlined}.}
\vspace{-10pt}
\begin{tabular}{lcc}
\toprule
\textbf{Method} & \textbf{Transductive} & \textbf{Inductive} \\
\midrule
JODIE      & 69.89 $\pm$ 0.21          & 66.51 $\pm$ 0.64 \\
DyRep      & 71.13 $\pm$ 0.52          & 64.72 $\pm$ 0.23 \\
TCL        & 70.23 $\pm$ 0.12          & 67.88 $\pm$ 0.43 \\
TGAT       & 64.98 $\pm$ 0.53          & 61.12 $\pm$ 0.17 \\
TGN        & 73.18 $\pm$ 0.32          & 62.29 $\pm$ 0.74 \\
GraphMixer & 72.88 $\pm$ 0.32          & 71.29 $\pm$ 0.34 \\
DyGFormer  & 75.16 $\pm$ 1.22          & 68.29 $\pm$ 2.38 \\
HTGN       & 79.87 $\pm$ 0.78          & 70.73 $\pm$ 0.83 \\
FreeDyG    & 79.98 $\pm$ 0.54          & \underline{71.97 $\pm$ 0.27} \\
HGWaveNet  & \underline{80.55 $\pm$ 0.48} & 71.42 $\pm$ 0.38 \\
% CTAN & 79.81 $\pm$ 0.42 & 66.05 $\pm$ 0.31 \\
% DyGMamba & 91.33 $\pm$ 0.31 & 86.89 $\pm$ 0.25 \\
\midrule
RLSTG      & \textbf{83.14 $\pm$ 0.21} & \textbf{73.26 $\pm$ 0.10} \\
\bottomrule
\end{tabular}
\label{tab:link_prediction_hist}
\end{table}

% \subsection{Node Classification}

\subsection{Ablation Studies}

\subsubsection{Ablation 1: The Impact of Geometric Space}

\begin{table*}[t]
  \centering
  \small
    \caption{Geometric Ablation on Node Regression and Link Prediction Tasks. The best results are in \textbf{bold} and the second best are \underline{underlined}, averaged over 5 separate runs.}
  \vspace{-10pt}
  \begin{tabular}{l|ccc||cc}
    \toprule
    & \multicolumn{3}{c||}{\textbf{Nodes Features Regression} (MAE$\downarrow$)} & \multicolumn{2}{c}{\textbf{Link Prediction} (AP$\uparrow$)} \\
    \cmidrule(lr){2-4} \cmidrule(lr){5-6}
    \textbf{Manifold} & \textbf{METR-LA} & \textbf{PEMS03} & \textbf{PEMS04} & \textbf{Transductive} & \textbf{Inductive} \\
    \midrule
    $\mathbb{E}^{17}$ & 2.85 $\pm$ 0.02 & 17.60 $\pm$ 0.02 & 24.21 $\pm$ 0.01 & 80.03 $\pm$ 0.23 & 71.05 $\pm$ 0.28 \\
    $\mathbb{E}^{16}$ & 2.86 $\pm$ 0.01 & 17.62 $\pm$ 0.02 & 24.17 $\pm$ 0.02 & 80.11 $\pm$ 0.37 & 71.15 $\pm$ 0.22 \\
    $\mathbb{S}^{16}$ & 2.85 $\pm$ 0.03 & \underline{17.51 $\pm$ 0.01} & 24.22 $\pm$ 0.04 & 76.24 $\pm$ 0.16 & 64.10 $\pm$ 0.25 \\
    $\mathbb{H}^{16}$ & 2.83 $\pm$ 0.02 & 17.52 $\pm$ 0.02 & 24.12 $\pm$ 0.02 & \textbf{83.14 $\pm$ 0.21} & \textbf{73.26 $\pm$ 0.11} \\
    \midrule
    $\mathbb{H}^{8}\times\mathbb{E}^{8}$ & \underline{2.80 $\pm$ 0.02} & 17.59 $\pm$ 0.02 & \underline{24.11 $\pm$ 0.03} & \underline{81.50 $\pm$ 0.25} & \underline{71.40 $\pm$ 0.19} \\
    $\mathbb{E}^{8}\times\mathbb{S}^{8}$ & \textbf{2.77 $\pm$ 0.01} & \textbf{17.49 $\pm$ 0.02} & 24.19 $\pm$ 0.02 & 79.40 $\pm$ 0.24 & 68.30 $\pm$ 0.28 \\
    $\mathbb{H}^{8}\times\mathbb{S}^{8}$ & 2.83 $\pm$ 0.02 & 17.55 $\pm$ 0.04 & \textbf{24.08 $\pm$ 0.01} & 81.03 $\pm$ 0.26 & 70.95 $\pm$ 0.17 \\
    \bottomrule
  \end{tabular}
  \label{tab:geo_ablation}
\end{table*}

\begin{table}[t]
\centering
\small
\caption{ODE Solver Ablation. We report MAE and relative computation time, normalized to the Euler method.}
\vspace{-10pt}
\begin{tabular}{lcc}
\toprule
\textbf{ODE Solver} & \textbf{MAE} $\downarrow$ & \textbf{Time (Relative)} \\
\midrule
Euler & 1.43 $\pm$ 0.20 & 1.0$\times$ \\
RK4 & 0.75 $\pm$ 0.27 & 4.5$\times$ \\
\midrule
GD (unfold=1) & 1.32 $\pm$ 0.07 & 1.5$\times$ \\
GD (unfold=4) & \textbf{0.46 $\pm$ 0.12} & 5.6$\times$ \\
\bottomrule
\end{tabular}
\label{tab:solver_comp}
\vspace{-5pt}
\end{table}

\begin{figure}[t]
    \centering
    \includegraphics[width=0.7\linewidth]{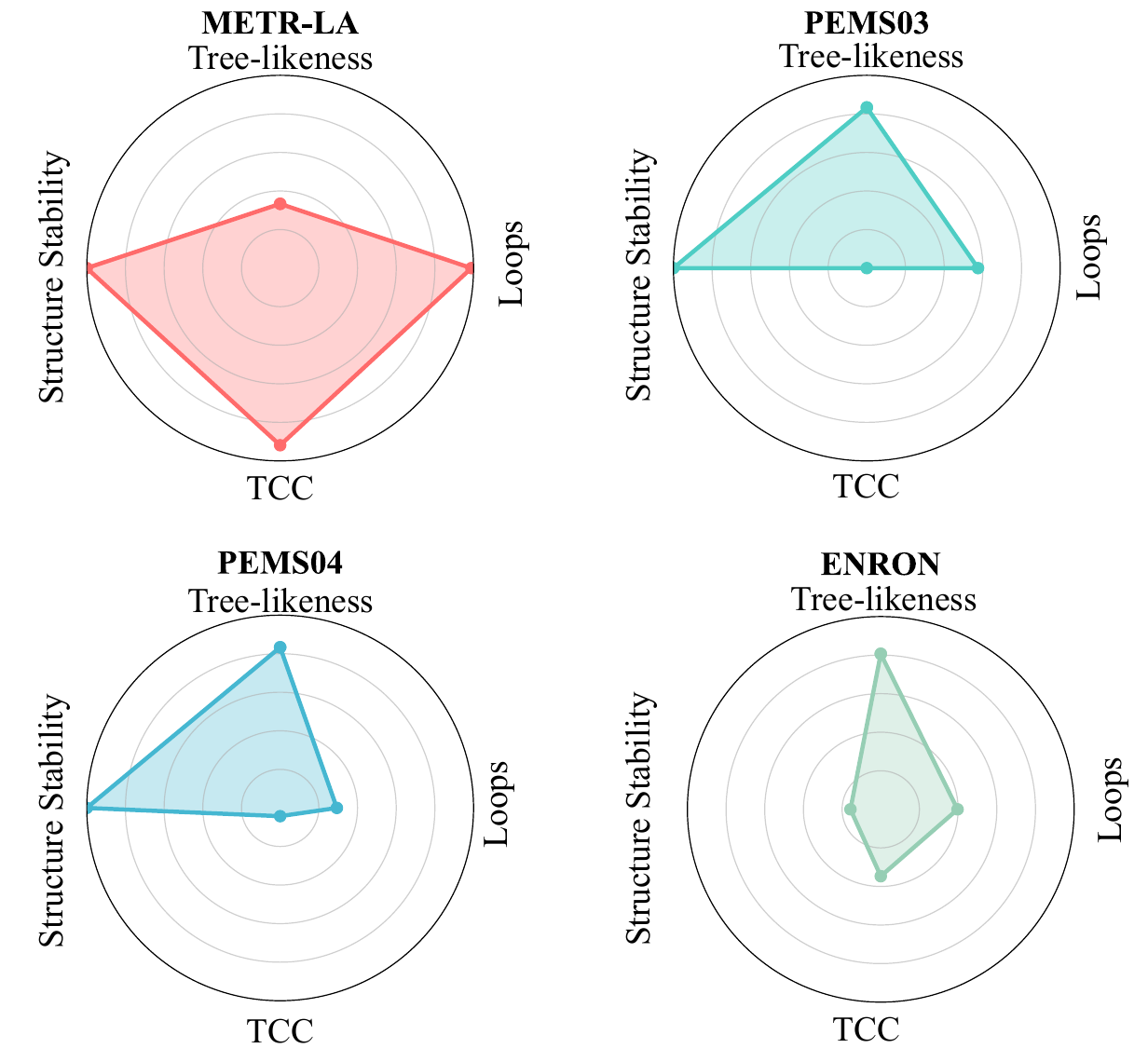}
    \vspace{-10pt}
    \caption{Geometric analysis of benchmark datasets. Tree-likeness is measured by $\frac{1}{\delta}$. Loops is measured by $\ln(1+\beta_1)$. \(\delta\) denotes Gromov hyperbolicity and \(\beta_{1}\) the first Betti number.}
    \label{fig:topo}
    \vspace{-10pt}
\end{figure}

To isolate and quantify the contribution of the geometric representation space, we conducted a rigorous ablation study, with results presented in Table \ref{tab:geo_ablation}. The analysis demonstrates that aligning the model's geometry with the data's intrinsic structure is more critical to performance than simply increasing representation capacity. Recognizing that real-world graphs rarely conform perfectly to specific spaces, the selection of an appropriate product manifold is guided by a quantitative analysis of the data's intrinsic topological properties, aiming to mitigate geometric distortion. This analysis is detailed in the Appendix~\ref{appendix:Geometric_Features} and visualized in Figure~\ref{fig:topo}, which reveals the distinct geometric signatures of the benchmark datasets and provides a basis for interpreting the ablation results.

The ablation results is in Table \ref{tab:geo_ablation}. Across \emph{METR-LA}, \emph{PEMS03}, \emph{PEMS04}, and \emph{ENRON}, the ablations validate our central claim that RLSTG reduces metric distortion when the embedding curvature is matched to measured topology, and that mixed-curvature products help only when signals are genuinely heterogeneous. \emph{METR-LA} exhibits pronounced cyclicality with only moderate hyperbolicity, hence $\mathbb{E}^{8}\!\times\!\mathbb{S}^{8}$ attains the best MAE, narrowly ahead of $\mathbb{H}^{8}\!\times\!\mathbb{E}^{8}$ and clearly outperforming single-curvature $\mathbb{H}^{16}$/$\mathbb{S}^{16}$, showing that explicitly modeling dense global/local loops is more beneficial than over-emphasizing a weak hierarchy. \emph{PEMS03} is near tree-like yet punctuated by large loops; here $\mathbb{E}^{8}\!\times\!\mathbb{S}^{8}$ slightly surpasses $\mathbb{S}^{16}$ and hyperbolic products, indicating that (i) faithfully encoding cycles via $\mathbb{S}$ is crucial, and (ii) the overall planar layout is better captured by $\mathbb{E}$ than by $\mathbb{H}$ despite low $\delta$. \emph{PEMS04} is almost purely hierarchical, so hyperbolic geometry is necessary and sufficient. Finally, \emph{ENRON} combines high topological volatility with a strong hierarchical backbone and moderate clustering. A single, correct inductive bias dominates, and pure $\mathbb{H}^{16}$ decisively wins for link prediction, while product spaces add no value. These findings collectively underscore that a tailored geometric approach is more critical than the representational capacity of a generic Euclidean space.

\subsubsection{Ablation 2: The Contribution of the GD Solver}

The efficacy of GD solver was empirically validated via an ablation study against manifold-generalized Euler and RK4 integrators (details in Appendix~\ref{appendix:other_ode}). The task involved integrating the Riemannian Liquid ODE on a $\mathbb{S}^2$ for 100 randomly initialized trajectories. We evaluated performance using MAE against a high-precision ground truth and relative computation time.

As shown in Table~\ref{tab:solver_comp}, the GD solver with four unfold steps (GD, unfold=4), a configuration that more precisely integrates the stiff ODE component, doubles the accuracy compared to RK4 with comparable computational overhead. In contrast, the Euler method is computationally efficient but highly inaccurate, while the accuracy of RK4 is offset by its substantial computational cost. Our solver thus provides a superior balance between accuracy and efficiency. A visual comparison in Figure \ref{fig:solver_com} provides a clear, intuitive confirmation of these numerical findings.

\begin{figure}[t]
    \centering
    \includegraphics[width=0.8\linewidth]{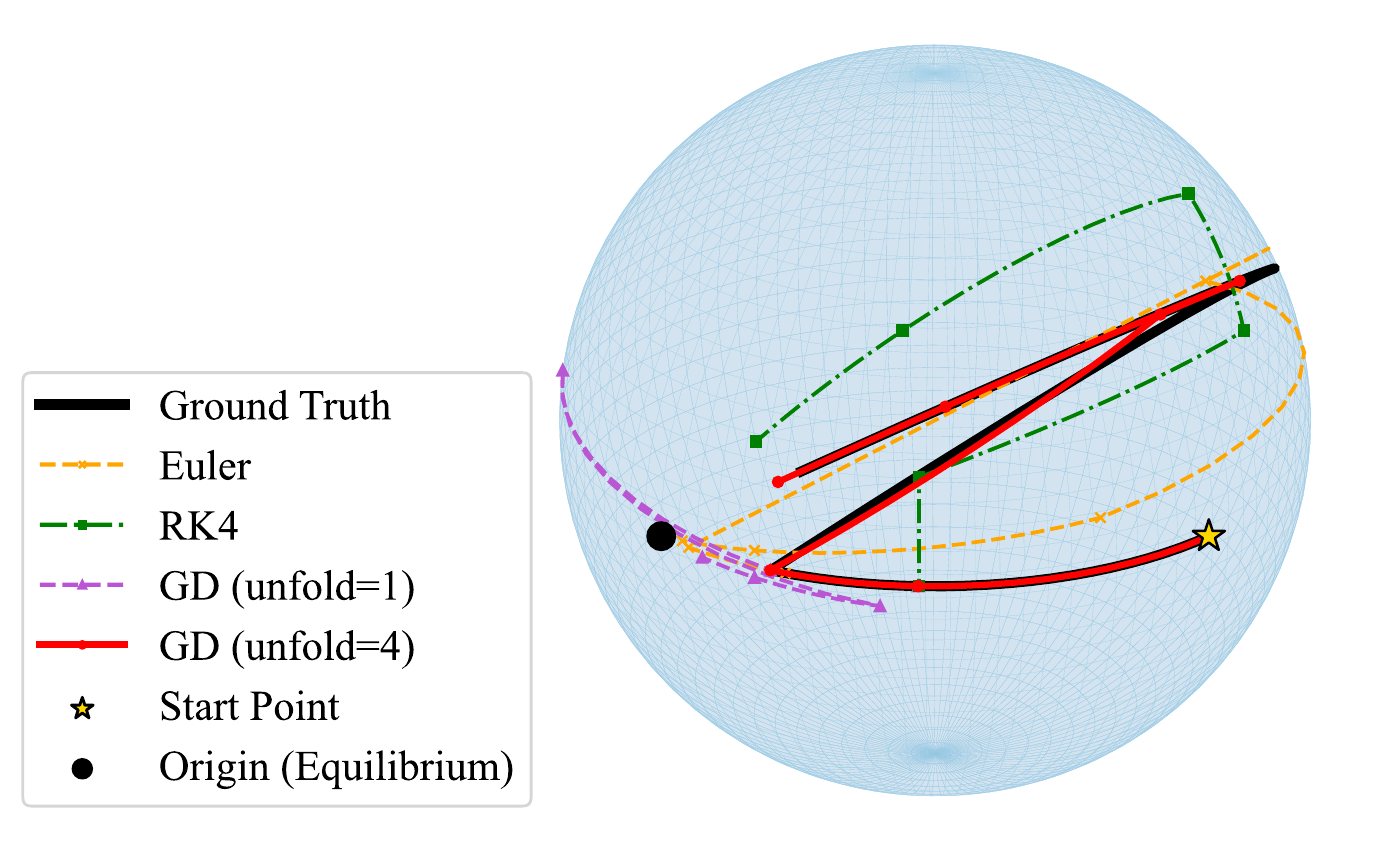}
    \vspace{-10pt}
    \caption{Trajectories on $\mathbb{S}^2$ for stiff system with decay phase (toward the black point) followed by driven phase (away from it). The trajectory generated by our  solver (red) is visually almost indistinguishable from the ground truth. It faithfully captures the intricate path along the sphere, demonstrating its ability to handle both the stiff decay forces and the driving input within the curved geometry with high precision and stability.}
    \label{fig:solver_com}
    \vspace{-15pt}
\end{figure}

\section{Conclusions, Impact and Future Work}

We introduced the Riemannian Liquid Spatio-Temporal Graph Network (RLSTG), a framework that unifies LTC dynamics with Riemannian geometry. RLSTG reduces the geometric distortion of Euclidean models by evolving dynamics on manifolds matched to the data's intrinsic structure. We extend theoretical guarantees of stability and universal approximation for LTCs to the Riemannian domain. Our work bridges continuous-time dynamics and geometric deep learning, enabling more faithful modeling of complex networks. Future research could address current limitations, including computational cost and fixed geometries, by exploring directions such as enhanced efficiency, automated manifold selection, and dynamic geometric representations.

\balance
\bibliographystyle{ACM-Reference-Format}
\bibliography{sample-base}

@article{chen2018neuralode,
  title={Neural Ordinary Differential Equations},
  author={Chen, Ricky T. Q. and Rubanova, Yulia and Bettencourt, Jesse and Duvenaud, David},
  journal={Advances in Neural Information Processing Systems},
  year={2018}
}

@article{rubanova2019latent,
  title={Latent ordinary differential equations for irregularly-sampled time series},
  author={Rubanova, Yulia and Chen, Ricky TQ and Duvenaud, David K},
  journal={Advances in neural information processing systems},
  volume={32},
  year={2019}
}

@article{kidger2020neural,
  title={Neural controlled differential equations for irregular time series},
  author={Kidger, Patrick and Morrill, James and Foster, James and Lyons, Terry},
  journal={Advances in neural information processing systems},
  volume={33},
  pages={6696--6707},
  year={2020}
}

@inproceedings{hasani2021liquid,
  title={Liquid time-constant networks},
  author={Hasani, Ramin and Lechner, Mathias and Amini, Alexander and Rus, Daniela and Grosu, Radu},
  booktitle={Proceedings of the AAAI Conference on Artificial Intelligence},
  volume={35},
  number={9},
  pages={7657--7666},
  year={2021}
}

@article{nickel2017poincare,
  title={Poincar{\'e} embeddings for learning hierarchical representations},
  author={Nickel, Maximillian and Kiela, Douwe},
  journal={Advances in neural information processing systems},
  volume={30},
  year={2017}
}

@article{ganea2018hyperbolic,
  title={Hyperbolic neural networks},
  author={Ganea, Octavian and B{\'e}cigneul, Gary and Hofmann, Thomas},
  journal={Advances in neural information processing systems},
  volume={31},
  year={2018}
}

@article{chami2019hyperbolic,
  title={Hyperbolic graph convolutional neural networks},
  author={Chami, Ines and Ying, Zhitao and R{\'e}, Christopher and Leskovec, Jure},
  journal={Advances in neural information processing systems},
  volume={32},
  year={2019}
}

@inproceedings{gravina2024tgode,
  title     = {Temporal Graph ODEs for Irregularly-Sampled Time Series},
  author    = {Gravina, Alessio and Zambon, Daniele and Bacciu, Davide and Alippi, Cesare},
  booktitle = {Proceedings of the Thirty-Third International Joint Conference on
           Artificial Intelligence, {IJCAI-24}},
  publisher = {International Joint Conferences on Artificial Intelligence Organization},
  editor    = {Kate Larson},
  pages     = {4025--4034},
  year      = {2024},
  month     = {8},
  note      = {Main Track},
  doi       = {10.24963/ijcai.2024/445},
  url       = {https://doi.org/10.24963/ijcai.2024/445},
}

@inproceedings{fang2021spatial,
  title={Spatial-temporal graph ode networks for traffic flow forecasting},
  author={Fang, Zheng and Long, Qingqing and Song, Guojie and Xie, Kunqing},
  booktitle={Proceedings of the 27th ACM SIGKDD conference on knowledge discovery \& data mining},
  pages={364--373},
  year={2021}
}

@inproceedings{li2018dcrnn_traffic,
  title={Diffusion Convolutional Recurrent Neural Network: Data-Driven Traffic Forecasting},
  author={Li, Yaguang and Yu, Rose and Shahabi, Cyrus and Liu, Yan},
  booktitle={International Conference on Learning Representations (ICLR '18)},
  year={2018}
}

@article{hornik1989multilayer,
  title={Multilayer feedforward networks are universal approximators},
  author={Hornik, Kurt and Stinchcombe, Maxwell and White, Halbert},
  journal={Neural networks},
  volume={2},
  number={5},
  pages={359--366},
  year={1989},
  publisher={Elsevier}
}

@article{funahashi1993approximation,
  title={Approximation of dynamical systems by continuous time recurrent neural networks},
  author={Funahashi, Ken-ichi and Nakamura, Yuichi},
  journal={Neural networks},
  volume={6},
  number={6},
  pages={801--806},
  year={1993},
  publisher={Elsevier}
}

@inproceedings{mostafa2022hyperbolic,
  title={Hyperbolic spatial temporal graph convolutional networks},
  author={Mostafa, Abdelrahman and Peng, Wei and Zhao, Guoying},
  booktitle={2022 IEEE International Conference on Image Processing (ICIP)},
  pages={3301--3305},
  year={2022},
  organization={IEEE}
}

@inproceedings{bachmann2020constant,
  title={Constant curvature graph convolutional networks},
  author={Bachmann, Gregor and B{\'e}cigneul, Gary and Ganea, Octavian},
  booktitle={International conference on machine learning},
  pages={486--496},
  year={2020},
  organization={PMLR}
}

@inproceedings{gu2018learning,
  title={Learning mixed-curvature representations in product spaces},
  author={Gu, Albert and Sala, Frederic and Gunel, Beliz and R{\'e}, Christopher},
  booktitle={International conference on learning representations},
  year={2018}
}

@inproceedings{marino2024liquid,
  title={Liquid-Graph Time-Constant Network for Multi-Agent Systems Control},
  author={Marino, Antonio and Pacchierotti, Claudio and Giordano, Paolo Robuffo},
  booktitle={Conference on decision and control 2024},
  year={2024}
}

@article{hasani2022closed,
  title={Closed-form continuous-time neural networks},
  author={Hasani, Ramin and Lechner, Mathias and Amini, Alexander and Liebenwein, Lucas and Ray, Aaron and Tschaikowski, Max and Teschl, Gerald and Rus, Daniela},
  journal={Nature Machine Intelligence},
  volume={4},
  number={11},
  pages={992--1003},
  year={2022},
  publisher={Nature Publishing Group UK London}
}

@inproceedings{sarkar2011low,
  title={Low distortion delaunay embedding of trees in hyperbolic plane},
  author={Sarkar, Rik},
  booktitle={International symposium on graph drawing},
  pages={355--366},
  year={2011},
  organization={Springer}
}

@book{burago2001course,
  title={A course in metric geometry},
  author={Burago, Dmitri and Burago, Yuri and Ivanov, Sergei and others},
  volume={33},
  year={2001},
  publisher={American Mathematical Society Providence}
}

@inproceedings{kumar2019predicting,
  title={Predicting dynamic embedding trajectory in temporal interaction networks},
  author={Kumar, Srijan and Zhang, Xikun and Leskovec, Jure},
  booktitle={Proceedings of the 25th ACM SIGKDD international conference on knowledge discovery \& data mining},
  pages={1269--1278},
  year={2019}
}

@inproceedings{trivedi2019dyrep,
  title={Dyrep: Learning representations over dynamic graphs},
  author={Trivedi, Rakshit and Farajtabar, Mehrdad and Biswal, Prasenjeet and Zha, Hongyuan},
  booktitle={International conference on learning representations},
  year={2019}
}

@article{wang2021tcl,
  title={Tcl: Transformer-based dynamic graph modelling via contrastive learning},
  author={Wang, Lu and Chang, Xiaofu and Li, Shuang and Chu, Yunfei and Li, Hui and Zhang, Wei and He, Xiaofeng and Song, Le and Zhou, Jingren and Yang, Hongxia},
  journal={arXiv preprint arXiv:2105.07944},
  year={2021}
}

@inproceedings{tgn_icml_grl2020,
    title={Temporal Graph Networks for Deep Learning on Dynamic Graphs},
    author={Emanuele Rossi and Ben Chamberlain and Fabrizio Frasca and Davide Eynard and Federico 
    Monti and Michael Bronstein},
    booktitle={ICML 2020 Workshop on Graph Representation Learning},
    year={2020}
}

@inproceedings{cong2023do,
title={Do We Really Need Complicated Model Architectures For Temporal Networks?},
author={Weilin Cong and Si Zhang and Jian Kang and Baichuan Yuan and Hao Wu and Xin Zhou and Hanghang Tong and Mehrdad Mahdavi},
booktitle={The Eleventh International Conference on Learning Representations },
year={2023},
url={https://openreview.net/forum?id=ayPPc0SyLv1}
}

@article{yu2023towards,
  title={Towards better dynamic graph learning: New architecture and unified library},
  author={Yu, Le and Sun, Leilei and Du, Bowen and Lv, Weifeng},
  journal={Advances in Neural Information Processing Systems},
  volume={36},
  pages={67686--67700},
  year={2023}
}

@inproceedings{tian2024freedyg,
  title={Freedyg: Frequency enhanced continuous-time dynamic graph model for link prediction},
  author={Tian, Yuxing and Qi, Yiyan and Guo, Fan},
  booktitle={The twelfth international conference on learning representations},
  year={2024}
}

@inproceedings{bai2023hgwavenet,
  title={HGWaveNet: A Hyperbolic Graph Neural Network for Temporal Link Prediction},
  author={Bai, Qijie and Nie, Changli and Zhang, Haiwei and Zhao, Dongming and Yuan, Xiaojie},
  booktitle={Proceedings of the ACM Web Conference 2023},
  pages={523--532},
  year={2023}
}

@InProceedings{gravina2024ctan,
    title = {Long Range Propagation on Continuous-Time Dynamic Graphs},
    author = {Gravina, Alessio and Lovisotto, Giulio and Gallicchio, Claudio and Bacciu, Davide and Grohnfeldt, Claas},
    booktitle = {Proceedings of the 41st International Conference on Machine Learning},
    pages = {16206--16225},
    year = {2024},
    editor = {Salakhutdinov, Ruslan and Kolter, Zico and Heller, Katherine and Weller, Adrian and Oliver, Nuria and Scarlett, Jonathan and Berkenkamp, Felix},
    volume = {235},
    series = {Proceedings of Machine Learning Research},
    month = {21--27 Jul},
    publisher = {PMLR},
    url = {https://proceedings.mlr.press/v235/gravina24a.html},
}

@inproceedings{choi2022STGNCDE,
  title={Graph Neural Controlled Differential Equations for Traffic Forecasting},
  author={Jeongwhan Choi AND Hwangyong Choi AND Jeehyun Hwang AND Noseong Park},
  booktitle={AAAI},
  year={2022}
}

@inproceedings{ GMAN-AAAI2020,
  author     = "Chuanpan Zheng and Xiaoliang Fan and Cheng Wang and Jianzhong Qi",
  title      = "GMAN: A Graph Multi-Attention Network for Traffic Prediction",
  booktitle  = "AAAI",
  pages      = "1234--1241",
  year       = "2020"
}

@article{bai2020adaptive,
  title={Adaptive graph convolutional recurrent network for traffic forecasting},
  author={Bai, Lei and Yao, Lina and Li, Can and Wang, Xianzhi and Wang, Can},
  journal={Advances in neural information processing systems},
  volume={33},
  pages={17804--17815},
  year={2020}
}

@article{kingma2014adam,
  title={Adam: A method for stochastic optimization},
  author={Kingma, Diederik P and Ba, Jimmy},
  journal={arXiv preprint arXiv:1412.6980},
  year={2014}
}

@inproceedings{wan2025rethink,
  title={Rethink graphode generalization within coupled dynamical system},
  author={Wan, Guancheng and Huang, Zijie and Zhao, Wanjia and Luo, Xiao and Sun, Yizhou and Wang, Wei},
  booktitle={Forty-second International Conference on Machine Learning},
  year={2025}
}

@inproceedings{yang2021discrete,
  title={Discrete-time temporal network embedding via implicit hierarchical learning in hyperbolic space},
  author={Yang, Menglin and Zhou, Min and Kalander, Marcus and Huang, Zengfeng and King, Irwin},
  booktitle={Proceedings of the 27th ACM SIGKDD conference on knowledge discovery \& data mining},
  pages={1975--1985},
  year={2021}
}

@article{kratsios2022universal-a,
  title={Universal approximation theorems for differentiable geometric deep learning},
  author={Kratsios, Anastasis and Papon, L{\'e}onie},
  journal={Journal of Machine Learning Research},
  volume={23},
  number={196},
  pages={1--73},
  year={2022}
}

@inproceedings{kratsios2022universal-b,
title={Universal Approximation Under Constraints is Possible with Transformers},
author={Anastasis Kratsios and Behnoosh Zamanlooy and Tianlin Liu and Ivan Dokmani{\'c}},
booktitle={International Conference on Learning Representations},
year={2022},
url={https://openreview.net/forum?id=JGO8CvG5S9}
}

@article{hairer2006structure,
  title={Structure-preserving algorithms for ordinary differential equations},
  author={Hairer, Ernst and Lubich, Christian and Wanner, Gerhard},
  journal={Geometric numerical integration},
  volume={31},
  year={2006},
  publisher={Springer-Verlag Berlin}
}

@article{stern2009implicit,
  title={Implicit-explicit variational integration of highly oscillatory problems},
  author={Stern, Ari and Grinspun, Eitan},
  journal={Multiscale Modeling \& Simulation},
  volume={7},
  number={4},
  pages={1779--1794},
  year={2009},
  publisher={SIAM}
}

@book{leimkuhler2004simulating,
  title={Simulating hamiltonian dynamics},
  author={Leimkuhler, Benedict and Reich, Sebastian},
  number={14},
  year={2004},
  publisher={Cambridge university press}
}

@inproceedings{guo2019attention,
  title={Attention based spatial-temporal graph convolutional networks for traffic flow forecasting},
  author={Guo, Shengnan and Lin, Youfang and Feng, Ning and Song, Chao and Wan, Huaiyu},
  booktitle={Proceedings of the AAAI conference on artificial intelligence},
  volume={33},
  number={01},
  pages={922--929},
  year={2019}
}

@article{benson2018simplicial,
  title={Simplicial closure and higher-order link prediction},
  author={Benson, Austin R and Abebe, Rediet and Schaub, Michael T and Jadbabaie, Ali and Kleinberg, Jon},
  journal={Proceedings of the National Academy of Sciences},
  volume={115},
  number={48},
  pages={E11221--E11230},
  year={2018},
  publisher={National Academy of Sciences}
}

@misc{xu2020inductive,
      title={Inductive Representation Learning on Temporal Graphs}, 
      author={Da Xu and Chuanwei Ruan and Evren Korpeoglu and Sushant Kumar and Kannan Achan},
      year={2020},
      eprint={2002.07962},
      archivePrefix={arXiv},
      primaryClass={cs.LG},
      url={https://arxiv.org/abs/2002.07962}, 
}

@inproceedings{hasaniliquid,
  title={Liquid Structural State-Space Models},
  author={Hasani, Ramin and Lechner, Mathias and Wang, Tsun-Hsuan and Chahine, Makram and Amini, Alexander and Rus, Daniela},
  booktitle={The Eleventh International Conference on Learning Representations}
}

\appendix

\setcounter{equation}{0} % 将公式计数器重置为0
\renewcommand{\theequation}{S\arabic{equation}}

\section{Basic Geometries}\label{Basic_Geo}
\emph{Hyperbolic (Lorentz) model.} For curvature $\kappa=-1/K<0$ with $K>0$, set $\mathbb H_K^n=\{x\in\mathbb R^{n+1}:\langle x,x\rangle_{\mathcal L}=-K,\,x_0>0\}$ with Minkowski product $\langle x,y\rangle_{\mathcal L}=-x_0y_0+\sum_{i=1}^n x_i y_i$ and origin $o=(\sqrt K,0,\ldots,0)$. Tangent spaces are $T_x\mathbb H_K^n=\{v:\langle x,v\rangle_{\mathcal L}=0\}$, and the ambient-to-tangent projection is $\mathrm{Proj}_{T_x \mathbb H_K^n}(u)=u+\frac{\langle x,u\rangle_{\mathcal L}}{K}\,x.$ The geodesic distance satisfies $d_{\mathcal L}^K(x,y)=\sqrt K\,\operatorname{arcosh}\!\Big(-\frac{\langle x,y\rangle_{\mathcal L}}{K}\Big),$ and with $\|v\|_{\mathcal L}=\sqrt{\langle v,v\rangle_{\mathcal L}}$ the exponential/logarithmic maps are $\exp_x^K(v)=\cosh\!\Big(\frac{\|v\|_{\mathcal L}}{\sqrt K}\Big)x+\sqrt K\,\sinh\!\Big(\frac{\|v\|_{\mathcal L}}{\sqrt K}\Big)\frac{v}{\|v\|_{\mathcal L}},\quad \log_x^K(y)=d_{\mathcal L}^K(x,y)\,\frac{\,y+\frac{\langle x,y\rangle_{\mathcal L}}{K}x\,}{\big\|y+\frac{\langle x,y\rangle_{\mathcal L}}{K}x\big\|_{\mathcal L}},$ and parallel transport along the minimizing geodesic is the closed form $P_{x\to y}(v)=v+\frac{\langle y,v\rangle_{\mathcal L}}{\,K-\langle x,y\rangle_{\mathcal L}\,}\,(x+y).$

\emph{Sphere.} For radius $R>0$, $\mathbb S_R^n=\{p\in\mathbb R^{n+1}:\|p\|^2=R^2\}$ with Euclidean product, origin $o=(R,0,\ldots,0)$, $T_p\mathbb S_R^n=\{u:\langle p,u\rangle=0\}$, projection $\mathrm{Proj}_{T_p \mathbb S_R^n}(u)=u-\frac{\langle p,u\rangle}{R^2}\,p,$ distance $d^R(p,q)=R\,\arccos\!\Big(\frac{\langle p,q\rangle}{R^2}\Big),$ maps $\exp_p(v)=\cos\!\Big(\frac{\|v\|}{R}\Big)p+R\sin\!\Big(\frac{\|v\|}{R}\Big)\frac{v}{\|v\|}$, $ \log_p(q)=\frac{\theta}{\sin\theta}\Big(q-\frac{\langle p,q\rangle}{R^2}p\Big)$, $\theta=\arccos\!\Big(\frac{\langle p,q\rangle}{R^2}\Big),$ and transport $P_{p\to q}(u)=u-\frac{\langle q,u\rangle}{\,R^2+\langle p,q\rangle\,}\,(p+q).$

\emph{Euclidean space.} For $\mathbb E^n$ with origin $0$, we use $d(x,y)=\|y-x\|,\quad \exp_x(v)=x+v,\quad \log_x(y)=y-x,\quad P_{x\to y}(u)=u.$

\section{Proof of Theorem 1}\label{pf:solver_error}
\begin{proof}
Let $F=A+B$ and define the GD-ODE step by Lie--Trotter splitting
$\Psi_{\Delta t}=\Psi^A_{\Delta t}\circ\Psi^B_{\Delta t}$, where
$\Psi^B_{\Delta t}(x)=\mathrm{Exp}_x(\Delta t\,B(x))$ and
$\Psi^A_{\Delta t}(y)=\mathrm{Exp}_o\!\big(e^{-\Delta t C_n}\odot\mathrm{Log}_o(y)\big)$ with
$C_n=1/\boldsymbol{\tau}+f(x_n)$ frozen over the step. Denote the exact flow by $\Phi^{A+B}_{\Delta t}$.
By the triangle inequality,
\begin{equation}
d\big(\Psi_{\Delta t},\Phi^{A+B}_{\Delta t}\big)\le
d\big(\Psi^A_{\Delta t}\!\circ\Psi^B_{\Delta t},\Phi^A_{\Delta t}\!\circ\Phi^B_{\Delta t}\big)
+d\big(\Phi^A_{\Delta t}\!\circ\Phi^B_{\Delta t},\Phi^{A+B}_{\Delta t}\big).
\end{equation}
The Lie--Trotter/BCH expansion on a Riemannian manifold gives
\begin{equation}
\mathrm{Log}_{\Phi^{A+B}_{\Delta t}(x)}\!\big(\Phi^A_{\Delta t}\!\circ\Phi^B_{\Delta t}(x)\big)
= -\tfrac{(\Delta t)^2}{2}\,[A,B]_{\mathrm{Lie}}(x)+O((\Delta t)^3),
\end{equation}
so the magnitude is $O((\Delta t)^2)$ with the principal contribution given by the Lie bracket.

For $\dot{x}=B(x)$, the explicit Riemannian Euler
$\Psi^B_{\Delta t}(x)=\mathrm{Exp}_x(\Delta t B(x))$ has local error $O((\Delta t)^2)$.
For $\dot{x}=A_n(x)=C_n\odot\mathrm{Log}_x(o)$, the decay step $\Psi^A_{\Delta t}$
is first-order consistent: on constant-curvature factors the differential of
$\mathrm{Exp}$ in the radial direction equals parallel transport and
$\mathrm{Log}_{y}(o)=-P_{o\to y}(\mathrm{Log}_o(y))$, so the first-order terms coincide and the discrepancy starts at $O((\Delta t)^2)$. Hence
\begin{equation}
d\big(\Psi^A_{\Delta t}\!\circ\Psi^B_{\Delta t},\Phi^A_{\Delta t}\!\circ\Phi^B_{\Delta t}\big)=O((\Delta t)^2).
\end{equation}
The local error vector reads
\begin{equation}
E_{\Delta t}
=\tfrac{(\Delta t)^2}{2}\,[A,B]_{\mathrm{Lie}}(x_n)
+\mathcal{E}_{\mathrm{geom}}(\Delta t^2)+O((\Delta t)^3),
\end{equation}
where $\mathcal{E}_{\mathrm{geom}}$ collects $O((\Delta t)^2)$ geometric terms from base-point mismatch of $\mathrm{Log}/\mathrm{Exp}$ and curvature. Summing $N=T/\Delta t$ steps yields a global error $O(\Delta t)$; therefore GD-ODE is first-order convergent. \qedhere
\end{proof}

\section{Proof of Theorem 2}\label{pf:tau}
\begin{proof}
Assume each factor manifold $\mathcal M_i$ is Hadamard (complete, simply connected, nonpositive curvature). Then $\exp,\log$ are globally well-defined, $d_i^2(\cdot,o_i)$ is smooth, and Levi-Civita parallel transport $P_{o_i\to x_i}$ is an isometry. All operations below are taken component-wise.

We consider the component dynamics
\begin{equation}\label{eq:dynamics}
\dot x_i(t)=\Big[\tfrac{1}{\boldsymbol{\tau}_i}+f_i(\cdot)\Big]\log_{x_i(t)}(o_i)\;+\;f_i(\cdot)\,P_{o_i\to x_i(t)}(V_i),
\end{equation}
with $f_i(\cdot)\in[0,1]$ and $V_i\in T_{o_i}\mathcal M_i$ fixed. We also assume time has been non-dimensionalized so that $f_i$ has units of a rate (equivalently, any constant gain has been absorbed into $f_i$).

Let $\psi_i(x):=\tfrac12 d_i^2(x,o_i)$ and $r_i(t):=d_i(x_i(t),o_i)=\|\log_{x_i(t)}(o_i)\|_i$. On a Hadamard manifold $\nabla_{x}\psi_i(x)=-\log_{x}(o_i)$, hence by the chain rule and \eqref{eq:dynamics},
\begin{equation}
\begin{aligned}
\dot\psi_i(t) 
&=\big\langle\nabla\psi_i(x_i),\dot x_i\big\rangle_i \\
&=-\Big(\tfrac{1}{\boldsymbol{\tau}_i}+f_i\Big)\,\|\log_{x_i}(o_i)\|_i^2
\;-\;f_i\,\big\langle \log_{x_i}(o_i),\,P_{o_i\to x_i}(V_i)\big\rangle_i.
\end{aligned}
\end{equation}
Using Cauchy--Schwarz and the isometry of parallel transport,
$\big|\!\langle \log_{x_i}(o_i),P_{o_i\to x_i}(V_i)\rangle_i\!\big|
\le r_i(t)\,\|V_i\|_i$, we obtain the differential inequality
\begin{equation}\label{eq:psi-ineq}
\dot\psi_i(t)\;\le\;-\alpha_i(t)\,r_i(t)^2\;+\;f_i(t)\,A_i\,r_i(t),
\; \alpha_i(t):=\tfrac{1}{\boldsymbol{\tau}_i}+f_i(t),\;\; A_i:=\|V_i\|_i.
\end{equation}
For $r_i(t)>0$, $r_i=\sqrt{2\psi_i}$ implies
\begin{equation}\label{eq:r-ineq}
\dot r_i(t)\;=\;\frac{\dot\psi_i(t)}{r_i(t)}\;\le\;-\alpha_i(t)\,r_i(t)+f_i(t)\,A_i.
\end{equation}
At $r_i(t)=0$, \eqref{eq:r-ineq} holds for the upper Dini derivative (and trivially by continuity). The comparison lemma for scalar differential inequalities then yields exponential contraction of $r_i$ toward the (possibly time-varying) equilibrium radius $r_i^\star(t):=\tfrac{f_i(t)A_i}{\alpha_i(t)}$ with instantaneous rate $\alpha_i(t)$. Therefore it is natural to define the (component-wise, instantaneous) system time constant as
\begin{equation}
(\boldsymbol{\tau}_{\mathrm{sys}})_i(t):=\alpha_i(t)^{-1}=\frac{1}{\tfrac{1}{\boldsymbol{\tau}_i}+f_i(t)}
=\frac{\boldsymbol{\tau}_i}{1+\boldsymbol{\tau}_i f_i(t)}.
\end{equation}

Since $f_i(t)\in[0,1]$ for all $t$, the map $f\mapsto\boldsymbol{\tau}_i/(1+\boldsymbol{\tau}_i f)$ is decreasing on $[0,1]$ and we obtain the sharp bounds, holding for all times and states,
\begin{equation}
\frac{\boldsymbol{\tau}_i}{1+\boldsymbol{\tau}_i}\;\le\;(\boldsymbol{\tau}_{\mathrm{sys}})_i(t)\;\le\;\tau_i.
\end{equation}

This proves the claim.
\end{proof}

\section{Proof of Theorem 3}\label{pf:state}
\begin{proof}
Assume each factor manifold $\mathcal M_i$ is Hadamard (complete, simply connected, $\mathrm{sec}\le 0$), so that $\exp,\log$ are globally well-defined, $d_i^2(\cdot,o_i)$ is smooth, and Levi-Civita parallel transport $P_{o_i\to x_i}$ is an isometry. All operations are taken component-wise. Time has been non-dimensionalized so that rates add:
\begin{equation}\label{eq:dynamics-T3}
\dot x_i(t)=\Big[\tfrac{1}{\boldsymbol{\tau}_i}+f_i(t)\Big]\log_{x_i(t)}(o_i)\;+\;f_i(t)\,P_{o_i\to x_i(t)}(V_i),
\end{equation}

where $f_i(\cdot)\in[0,1]$ and $V_i\in T_{o_i}\mathcal M_i$. Let $\psi_i(x):=\tfrac12 d_i^2(x,o_i)$ and $r_i(t):=d_i(x_i(t),o_i)=\|\log_{x_i(t)}(o_i)\|_i$. On a Hadamard manifold, $\nabla\psi_i(x)=-\log_x(o_i)$, hence by the chain rule and \eqref{eq:dynamics-T3},
\begin{equation}
\begin{aligned}
\dot\psi_i(t)=\big\langle\nabla\psi_i(x_i),\dot x_i\big\rangle_i
=&-\Big(\tfrac{1}{\boldsymbol{\tau}_i}+f_i\Big)\|\log_{x_i}(o_i)\|_i^2 \\
&-f_i\,\big\langle \log_{x_i}(o_i),\,P_{o_i\to x_i}(V_i)\big\rangle_i.
\end{aligned}
\end{equation}

By Cauchy--Schwarz and the isometry of parallel transport,
$\big|\!\langle \log_{x_i}(o_i),P_{o_i\to x_i}(V_i)\rangle_i\!\big|
\le r_i(t)\,\|V_i\|_i$. Writing $A_i:=\|V_i\|_i$ and $\alpha_i(t):=\tfrac{1}{\boldsymbol{\tau}_i}+f_i(t)$, we obtain the differential inequality
\begin{equation}\label{eq:ri-ineq}
\dot r_i(t)=\frac{\dot\psi_i(t)}{r_i(t)}\;\le\;-\alpha_i(t)\,r_i(t)+f_i(t)\,A_i,
\end{equation}
where $r_i(t)>0$ and the upper Dini derivative at $r_i=0$. Since $f_i(t)\in[0,1]$, we have $\alpha_i(t)\ge 1/\boldsymbol{\tau}_i$ and $f_i(t)A_i\le A_i$, so by comparison
\begin{equation}\label{eq:comparison}
\dot r_i(t)\;\le\;-\tfrac{1}{\boldsymbol{\tau}_i}\,r_i(t)+A_i.
\end{equation}
Solving the linear comparison system yields, for all $t\ge 0$,
\begin{equation}\label{eq:explicit-bound}
r_i(t)\;\le\;e^{-t/\boldsymbol{\tau}_i}\,r_i(0)\;+\;\boldsymbol{\tau}_i A_i\big(1-e^{-t/\boldsymbol{\tau}_i}\big)
\;\le\;\max\{\,r_i(0),\;\boldsymbol{\tau}_i A_i\,\}.
\end{equation}

Define the (component-wise) invariant radii $R_i:=\boldsymbol{\tau}_i\|V_i\|_i$. From \eqref{eq:explicit-bound} we obtain:
(i) the closed ball $B_i:=\{x\in\mathcal M_i:\,d_i(x,o_i)\le R_i\}$ is forward invariant (if $r_i(0)\le R_i$ then $r_i(t)\le R_i$ for all $t$);
(ii) in general $r_i(t)\le \max\{r_i(0),R_i\}$ for all $t$, and $r_i(t)\to r_i^\star(t):=\frac{f_i(t)A_i}{\alpha_i(t)}\in[0,R_i]$ with instantaneous exponential rate $\alpha_i(t)$.
Passing to the product manifold with product metric, the set:$\mathcal B:=\prod_{i=1}^k B_i,$ which is forward invariant, and every trajectory is uniformly bounded:
\begin{equation}
d\big(x(t),o\big)\;\le\;\Big(\sum_{i=1}^k \max\{r_i(0),R_i\}^2\Big)^{1/2}\!.
\end{equation}
In particular, if $x(0)\in\mathcal B\text{ then }d\big(x(t),o\big)\le
\Big(\sum_{i=1}^k R_i^2\Big)^{1/2}\!$, proves boundedness and forward invariance as claimed.
\end{proof}

\section{Proof of Theorem 4}\label{appendix:UAT}
We consider autonomous dynamics ($I\equiv 0$).

\subsection{Preparatory Lemmas}
\begin{lemma}
For any smooth curve $\gamma$ and tangent vectors $a,b\in T_{\gamma(0)}\mathbb P$, the Levi--Civita parallel transport $P_{\gamma(0)\to\gamma(1)}$ is linear:
$P(a+b)=P(a)+P(b)$.
\end{lemma}
\begin{proof}
Parallel transport solves a linear ODE $\nabla_{\dot\gamma}X=0$ with linear initial data; linearity follows.
\end{proof}

\begin{lemma}
For any $o,x\in\mathbb P$,
\begin{equation}
-P_{o\to x}\big(\mathrm{Log}_{o}(x)\big)=\mathrm{Log}_{x}(o).
\end{equation}

\end{lemma}
\begin{proof}
Let $\gamma$ be the minimizing geodesic with $\gamma(0)=o$, $\gamma(1)=x$ and initial velocity $v=\mathrm{Log}_{o}(x)$. Transporting $v$ along $\gamma$ gives the terminal velocity; reversing $\gamma$ flips the sign, hence the identity.
\end{proof}

\begin{lemma}
Let $F,\widetilde F$ be Lipschitz on a geodesically convex domain with Lipschitz constant $L$ for $F$. If $\|F-\widetilde F\|\le\varepsilon_\ell$, then the solutions $\mathbf u,\mathbf x$ to $\dot{\mathbf u}=F(\mathbf u)$ and $\dot{\mathbf x}=\widetilde F(\mathbf x)$ with the same initial state satisfy
\begin{equation}
d(\mathbf u(t),\mathbf x(t))\le \frac{\varepsilon_\ell}{L}\,(e^{L|t-t_0|}-1).
\end{equation}

\end{lemma}

\begin{lemma}[Universal approximation, vector form]
If $H:\mathcal K\to\mathbb R^d$ is continuous on a compact $\mathcal K\subset\mathbb R^n$, then for any $\delta>0$ there exist gates $g_m:\mathcal K\to[0,1]$ and vectors $A^{(m)}\in\mathbb R^d$ such that
\begin{equation}
\sup_{v\in\mathcal K}\Big\|H(v)-\sum_{m=1}^M g_m(v)\,A^{(m)}\Big\|<\delta.
\end{equation}

\end{lemma}

\subsection{Main Theorem}
\begin{theorem*}
Let $\mathbb P=\prod_{i=1}^k\mathcal M_i$ be a product of complete constant-curvature manifolds. Let $F:S\subset\mathbb P\to T\mathbb P$ be $C^1$, bounded, and $L$-Lipschitz on $S$. Suppose the solution $\mathbf u$ to $\dot{\mathbf u}=F(\mathbf u)$ with $\mathbf u(0)=\mathbf x_0$ is unique on $[0,T]$ and $\mathbf u([0,T])\subset D\Subset S$. Then, for any $\varepsilon>0$ and partition $0=t_0<\dots<t_N=T$, there exists an RLSTG with Geodesic-Decay (GD) solver generating $\{\mathbf x_k\}_{k=0}^N$ such that
\begin{equation}
\max_{0\le k\le N} d(\mathbf x_k,\mathbf u(t_k))<\varepsilon.
\end{equation}

\end{theorem*}

\begin{proof}[Proof]
Let $\eta\in(0,1)$ be such that $\overline{D_\eta}\Subset S$, where $D_\eta:=\{x:\,d(x,D)<\eta\}$. Choose $o\in\mathbb P$ and (shrinking $\eta$ if needed) assume $\overline{D_\eta}$ lies in a normal neighborhood of $o$; thus $\mathrm{Log}_{o}$ and $\mathrm{Exp}_{o}$ are inverse diffeomorphisms between $\overline{D_\eta}$ and $\mathcal K:=\mathrm{Log}_{o}(\overline{D_\eta})\subset\mathbb R^d$, and the minimizing geodesic from $o$ to $x$ is unique.

Define a continuous map
\begin{equation}
H(v):=P_{\mathrm{Exp}_{o}(v)\to o}\!\big(F(\mathrm{Exp}_{o}(v))\big),\qquad v\in\mathcal K.
\end{equation}

Pick
$\varepsilon_\ell<\frac{L\,\min\{\varepsilon,\eta\}}{2\,(e^{LT}-1)}.$

By the vector UAT, there exist gates $g_m$ and vectors $A^{(m)}\in T_o\mathbb P\cong\mathbb R^d$ such that
\begin{equation}
\sup_{v\in\mathcal K}\Big\|H(v)-\sum_{m=1}^M g_m(v)\,A^{(m)}\Big\|<\frac{\varepsilon_\ell}{2}.
\end{equation}

Let $v=\mathrm{Log}_{o}(x)$ and set
\begin{equation}
\widetilde F(x):=-\frac{1}{\boldsymbol{\tau}}\,\mathrm{Log}_{x}(o)\;+\;\sum_{m=1}^M g_m\!\big(\mathrm{Log}_{o}(x)\big)\,P_{o\to x}\!\big(A^{(m)}\big),
\end{equation}

i.e., an RLSTG with a linear output layer $\sum_m g_m(\cdot)A^{(m)}$. Transporting to $T_o\mathbb P$ and using the two lemmas,
\begin{equation}
P_{x\to o}\widetilde F(x)=-\frac{1}{\boldsymbol{\tau}}\,\mathrm{Log}_{o}(x)+\sum_{m} g_m(v)\,A^{(m)}.
\end{equation}

Let $R:=\sup_{x\in D_\eta}\|\mathrm{Log}_{o}(x)\|<\infty$ and choose $\boldsymbol{\tau}$ so that $R/\boldsymbol{\tau}<\varepsilon_\ell/2$. Then, for all $x\in D_\eta$,
\begin{equation}
\|F(x)-\widetilde F(x)\|_{x}=\|P_{x\to o}F(x)-P_{x\to o}\widetilde F(x)\|
<\frac{\varepsilon_\ell}{2}+\frac{\varepsilon_\ell}{2}=\varepsilon_\ell.
\end{equation}

Smoothness on the compact $\overline{D_\eta}$ implies $\widetilde F$ is Lipschitz there.

Let $\mathbf x$ solve $\dot{\mathbf x}=\widetilde F(\mathbf x)$ with $\mathbf x(0)=\mathbf x_0$. By the Riemannian Gronwall lemma on the geodesically convex $D_\eta$,
\begin{equation}
d(\mathbf u(t),\mathbf x(t))\le \frac{\varepsilon_\ell}{L}\,(e^{Lt}-1)\le \frac{\min\{\varepsilon,\eta\}}{2}\quad (0\le t\le T).
\end{equation}

Hence $\mathbf x([0,T])\subset D_\eta$ (domain staying), closing the bootstrap.

The first-order GD solver applied to $\dot{\mathbf x}=\widetilde F(\mathbf x)$ yields
\begin{equation}
\max_{0\le k\le N} d\big(\mathbf x_k,\mathbf x(t_k)\big)\le C\,\Delta t_{\max},
\end{equation}

with $C$ depending on $T$, Lipschitz bounds of $\widetilde F$, and geometry on $D_\eta$. Choose the mesh so that $C\,\Delta t_{\max}<\min\{\varepsilon,\eta\}/2$. Then
\begin{equation}
\begin{aligned}
d(\mathbf x_k,\mathbf u(t_k)) &\le d(\mathbf x_k,\mathbf x(t_k))+d(\mathbf x(t_k),\mathbf u(t_k)) \\
&<\frac{\min\{\varepsilon,\eta\}}{2}+\frac{\min\{\varepsilon,\eta\}}{2}\le \varepsilon,
\end{aligned}
\end{equation}

for all $k$, as claimed.
\end{proof}

\section{Other ODE Solver on Manifolds}\label{appendix:other_ode}

\subsection{Euler and RK4 on Manifolds}
We integrate the ODE on a Riemannian manifold $\mathcal M$,
$\dot x(t)=F(x(t))\in T_{x(t)}\mathcal M$,
assuming a normal neighborhood so that $\mathrm{Exp}_x$ and the Levi–Civita parallel transport $P_{y\to x}:T_y\mathcal M\to T_x\mathcal M$ are well-defined along the minimizing geodesic. (On product manifolds, apply component-wise.)

\paragraph{Explicit Euler (first order).}
With step size $h>0$,
$x_{k+1}\;=\;\mathrm{Exp}_{x_k}\!\big(h\,F(x_k)\big).$

\paragraph{Geometric RK4 (fourth order).}
All vector sums are taken in $T_{x_k}\mathcal M$; intermediate evaluations are transported back to $T_{x_k}\mathcal M$:
\begin{align}
\alpha_1 &:= F(x_k),\\
x_k^{(2)} &:= \mathrm{Exp}_{x_k}\!\Big(\tfrac{h}{2}\,\alpha_1\Big),\quad
\alpha_2 := P_{x_k^{(2)}\to x_k}\big(F(x_k^{(2)})\big),\\
x_k^{(3)} &:= \mathrm{Exp}_{x_k}\!\Big(\tfrac{h}{2}\,\alpha_2\Big),\quad
\alpha_3 := P_{x_k^{(3)}\to x_k}\big(F(x_k^{(3)})\big),\\
x_k^{(4)} &:= \mathrm{Exp}_{x_k}\!\big(h\,\alpha_3\big),\quad
\alpha_4 := P_{x_k^{(4)}\to x_k}\big(F(x_k^{(4)})\big),\\[2pt]
x_{k+1} &:= \mathrm{Exp}_{x_k}\!\Big(\tfrac{h}{6}\big(\alpha_1+2\alpha_2+2\alpha_3+\alpha_4\big)\Big).
\end{align}
This respects manifold structure by: (i) evaluating $F$ at displaced points via $\mathrm{Exp}$, (ii) transporting all stage vectors to the common space $T_{x_k}\mathcal M$ before averaging.

\section{Experiment Setting}\label{appendix:experiment}

% \begin{table*}[h]
%   \centering
%   \small
%   \caption{Average Precision $\uparrow$ for temporal link-prediction results on Wikipedia datasets with random (rand) negative sampling strategy. Bold marks the best score within each column.}
%   \label{tab:link_prediction_rand}
%   \begin{tabular}{lcc}
%     \toprule
%     \textbf{Method} & \textbf{Transductive} & \textbf{Inductive} \\
%     \midrule
%     JODIE & $95.54 \pm 0.25$ & $92.15 \pm 0.60$ \\
%     DyRep & $95.92 \pm 0.55$ & $91.34 \pm 0.28$ \\
%     TCL & $96.91 \pm 0.15$ & $96.52 \pm 0.40$ \\
%     TGAT & $96.73 \pm 0.50$ & $95.88 \pm 0.20$ \\
%     TGN & $98.95 \pm 0.30$ & $97.01 \pm 0.70$ \\
%     CAWN & $98.66 \pm 0.25$ & $97.59 \pm 0.80$ \\
%     GraphMixer & $97.67 \pm 0.30$ & $96.03 \pm 0.31$ \\
%     DyGFormer & $98.99 \pm 1.20$ & $97.15 \pm 2.30$ \\
%     FreeDyG & $98.76 \pm 0.50$ & $97.81 \pm 0.30$ \\
%     HGWaveNet & $-- \pm --$ & $-- \pm --$ \\
%     CTAN & $-- \pm --$ & $-- \pm --$ \\
%     DyGMamba & $-- \pm --$ & $-- \pm --$ \\
%     \midrule
%     RLSTG & $\mathbf{98.88 \pm 0.19}$ & $\mathbf{98.95 \pm 0.15}$ \\
%     \bottomrule
%   \end{tabular}
% \end{table*}

\subsection{Hyperparameters and Parameter Counts}

We tuned the hyperparameters based on the model's performance on the validation set. For each hyperparameter, we performed a grid search over the ranges specified in Table \ref{tab:hyperparameters} and selected the value that yielded the best validation metric (i.e., the lowest MAE for regression tasks and the highest AP for link prediction tasks on validation set). The final optimal values are noted in the table.

\begin{table}[h]
\centering
\small
\caption{Hyperparameters used for the experimental evaluations}
\label{tab:hyperparameters}
\begin{tabular}{m{2cm}|c|m{4cm}} 
\toprule
\textbf{Parameter} & \textbf{Value} & \textbf{Description} \\
\midrule
Number of hidden units & 16 & Total node-embedding dimension.\\ 
Hyperboloid curvature & 1e-8 - 1.0 & Best: 1e-6. \\
Sphere curvature & 1.0 & \\
Batch size & 200 & \\
Learning rate & 1e-4 - 1e-2 & Best: 1e-3. \\
ODE-solver step & 1/4 & Unfold = 4. \\
Optimizer & Adam  & \cite{kingma2014adam} \\
Weight Decay & 1e-8 - 1e-1 & Adam parameter; Best: 1e-6.\\
Dropout & 0 - 0.2 & Best: 0. \\
BPTT length & 1 & BPTT steps. \\
Validation evaluation interval & 1 & Validate every x epochs. \\
Training epochs & 500 & \\
Seed & 42 & Fixed in search; random across 5 runs.\\
\bottomrule
\end{tabular}
\end{table}

\subsection{Computational Cost Analysis}

\begin{table}[h]
\centering
\caption{Compute cost vs. baselines on METR-LA.}
\label{tab:computation_comparison}
\resizebox{\columnwidth}{!}{%
\begin{tabular}{l c r r r r}
\toprule
Method & Manifold & Mem (MB) & Time/ep (s) & Params (k) & FLOPs/pass (M) \\
\midrule
STGODE & $\mathbb{E}^{16}$ & 94.3 & 40.0 & 1.651 & 0.437 \\
\midrule
\multirow{7}{*}{\textbf{RLSTG}} 
& $\mathbb{E}^{17}$                   & 171.3 & 15.1 & 2.228 & 3.120 \\
& $\mathbb{E}^{16}$                   & 164.7 & 14.6 & 1.985 & 2.764 \\
& $\mathbb{S}^{16}$                   & 212.5 & 47.2 & 2.228 & 3.120 \\
& $\mathbb{H}^{16}$                   & 235.4 & 71.1 & 2.228 & 3.120 \\
& $\mathbb{H}^{8}\!\times\!\mathbb{E}^{8}$ & 214.6 & 75.9 & 2.228 & 3.120 \\
& $\mathbb{E}^{8}\!\times\!\mathbb{S}^{8}$ & 201.0 & 55.6 & 2.228 & 3.120 \\
& $\mathbb{H}^{8}\!\times\!\mathbb{S}^{8}$ & 244.0 & 80.1 & 2.485 & 3.497 \\
\bottomrule
\end{tabular}%
}
\end{table}

To provide a comprehensive view of our model's practical implications, we analyzed its computational cost against baseline models, with the results detailed in Table \ref{tab:computation_comparison}. All experiments were conducted on a server equipped with a single NVIDIA RTX 4090 GPU (24GB of VRAM), a 20-vCPU Intel Xeon Platinum 8470Q processor, and 90GB of system memory. The software environment was configured with Ubuntu 22.04, running PyTorch 2.5.1, Python 3.12, and CUDA 12.4. The results indicate that RLSTG, while more powerful, carries a higher computational overhead compared to the Euclidean-based STGODE. The increased cost of RLSTG is an expected trade-off, arising from the richer representations and the additional computations required for manifold operations, such as exponential maps and parallel transport. This increased cost is the price for the model's enhanced flexibility and superior performance on datasets with heterogeneous geometric structures, as demonstrated in our experimental results.

\subsection{Metrics Analysis}\label{appendix:Geometric_Features}
\begin{table}[h!]
\centering
\small
\caption{A comparison of metrics across different datasets, rounded to four decimal places.}
\label{tab:metrics_four_decimals}
\begin{tabular}{lcccc}
\toprule
\textbf{Metric} & \textbf{METR-LA} & \textbf{PEMS03} & \textbf{PEMS04} & \textbf{ENRON} \\
\midrule

Nodes &207 &358 &307 &184 \\
Edges &1515 &442 &209 &125235 \\
Node/Link Feature &1/1 &1/1 &1/1 &0/0\\
Duration &4 months &3 months &2 months & 3 years \\
Topology &Static &Static &Static &Dynamic \\

\midrule
$\delta_\text{mean}$ & 2.5000 & 1.0000 & 1.0000 & 1.0333 \\
$\delta_\text{variance}$ & 0.0000 & 0.0000 & 0.0000 & 0.1489 \\
$\delta_\text{cv}$ & 0.0000 & 0.0000 & 0.0000 & 0.3734 \\
\midrule
$\beta_{1, \text{mean}}$ & 1108.0000 & 58.0000 & 7.0000 & 15.7111 \\
$\beta_{1, \text{variance}}$ & 0.0000 & 0.0000 & 0.0000 & 65.4943 \\
$\beta_{1, \text{cv}}$ & 0.0000 & 0.0000 & 0.0000 & 0.5151 \\
\midrule
$TCC_\text{mean}$ & 0.5512 & 0.0000 & 0.0256 & 0.2081 \\
$TCC_\text{variance}$ & 0.0000 & 0.0000 & 0.0000 & 0.0053 \\
$TCC_\text{90th percentile}$ & 0.5512 & 0.0000 & 0.0256 & 0.3031 \\
\midrule
Structure Stability & 1.0000 & 1.0000 & 1.0000 & 0.1567 \\
\bottomrule
\end{tabular}
\end{table}

We report standard graph metrics and their temporal summaries in Table~\ref{tab:metrics_four_decimals}.
$\delta$-hyperbolicity follows the four-point condition (Gromov);
$\beta_1$ is the first Betti number on graphs ($\beta_1=|E|-|V|+|C|$);
the global clustering coefficient $C$ is $3N_\triangle/N_3$.
Temporal statistics (mean/variance and CV) are computed over snapshots.
\textbf{TCC}$(\Delta t)$ counts temporally ordered triangle closures within delay $\Delta t$.
\textbf{Zigzag-PH} uses 1D barcodes over snapshots; we report $L_{1,\max}=\max_i(d_i-b_i)$.

\end{document}